\documentclass{article}

\usepackage{microtype}
\usepackage{graphicx}
\usepackage{subfigure}
\usepackage{booktabs} % for professional tables
\usepackage{multirow}
\usepackage{color}
\usepackage{amsthm}
\usepackage{amsmath} % assumes amsmath package installed
\usepackage{amssymb}  % assumes amsmath package installed
\usepackage{hyperref}

\usepackage[accepted]{icml2019}

\usepackage{algorithm}
\usepackage{algpseudocode}
\renewcommand{\algorithmicrequire}{\textbf{Input:}}
\renewcommand{\algorithmicensure}{\textbf{Output:}}

\usepackage[utf8]{inputenc}
\usepackage[english]{babel}
\newtheorem{theorem}{Theorem}[section]

\newtheorem{proposition}[theorem]{Proposition}
\newtheorem{lemma}[theorem]{Lemma}
\icmltitlerunning{Breaking Inter-Layer Co-Adaptation by Classifier Anonymization}

\newcommand{\etal}{\textit{et al}. }
\newcommand{\ie}{\textit{i}.\textit{e}.}

\newcommand{\argmin}{\mathop{\rm arg~min}\limits}

\usepackage{comment}

\begin{document}

\twocolumn[
\icmltitle{Breaking Inter-Layer Co-Adaptation by Classifier Anonymization}




\begin{icmlauthorlist}
\icmlauthor{Ikuro Sato}{itlab}
\icmlauthor{Kohta Ishikawa}{itlab}
\icmlauthor{Guoqing Liu}{itlab}
\icmlauthor{Masayuki Tanaka}{AIST}
\end{icmlauthorlist}

\icmlaffiliation{itlab}{Denso IT Laboratory, Inc., Japan}
\icmlaffiliation{AIST}{National Institute of Advanced Industrial Science and Technology, Japan}

\icmlcorrespondingauthor{Ikuro Sato}{isato@d-itlab.co.jp}
%\icmlcorrespondingauthor{Masayuki Tanaka}{masayuki.tanaka@aist.go.jp}

\icmlkeywords{Machine Learning, ICML, Deep Neural Network}

\vskip 0.3in
]



\printAffiliationsAndNotice{}  % leave blank if no need to mention equal contribution

\begin{abstract}
This study addresses an issue of co-adaptation between a feature extractor and a classifier in a neural network. A na\"ive joint optimization of a feature extractor and a classifier often brings situations in which an excessively complex feature distribution adapted to a very specific classifier degrades the test performance. We introduce a method called Feature-extractor Optimization through Classifier Anonymization (FOCA), which is designed to avoid an explicit co-adaptation between a feature extractor and a particular classifier by using many randomly-generated, weak classifiers during optimization. We put forth a mathematical proposition that states the FOCA features form a point-like distribution within the same class in a class-separable fashion under special conditions. Real-data experiments under more general conditions provide supportive evidences.
%
%One such example is that
%a classifier can generalize well even when it learns the FOCA features 
%of a partial dataset consisting of only 
%one data sample per class.
%
%An approximated geodesic distance between
%the classifier parameters trained by the full dataset
%and those by an orders-of-magnitude smaller partial dataset
%is orders-of-magnitude smaller than other methods.
%that employ joint optimizations 
%between the feature-extractor and the classifier parameters.
\end{abstract}

%%%%%%%%%%%%%%%%%%%%%%%%%%%%%%%%%%%%%%%%%%%%%%%%%%%%%%%%%%%%%%%%%%%%%%%%%%%%%%%%
\section{Introduction}
\label{one}

\begin{comment}
A typical neural network designed for 
an image classification task consists of
a local feature extractor consisting of convolutional layers often accompanied with spatial pooling layers
and a classifier consisting of fully-connected layers~\cite{LeCun1989}.
Much efforts have been paid for designing rich architectures of feature extractor;
for instance, skip-connections~\cite{He2016,Li2018}, network-in-network~\cite{Lin2013network}, 
restricted connectivity between layers~\cite{Szegedy2015} to name a few, are 
the standard techniques now.
As far as a feature extractor that has a powerful description ability is used,
a classifier of a relatively simple form works effectively~\cite{Hoffer2018,Lin2013network} 
in most cases.
\end{comment}

\begin{figure}[t]
    \centering
    {\small
    \begin{tabular}{cc}
    \hspace{-1.0mm}\includegraphics[height=31mm]{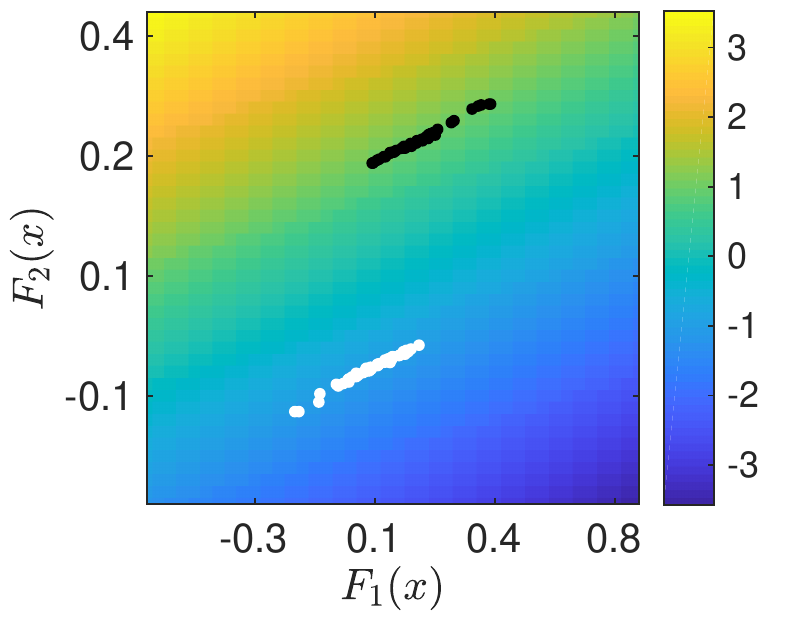} & 
    \hspace{-6.0mm}\includegraphics[height=31mm]{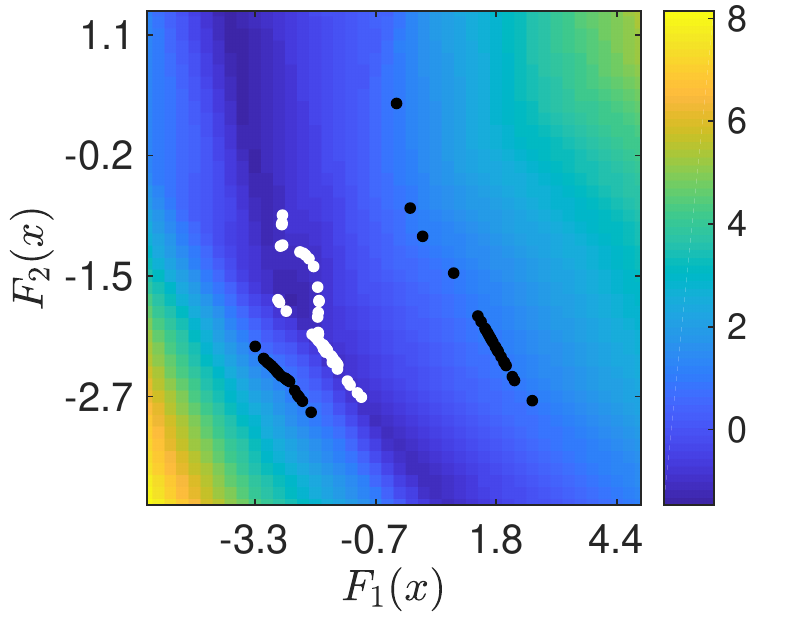} \\
    \multicolumn{2}{c}{(a) Joint opt. Single- (left) and multi-layered classifiers (right).} \\
    \hspace{-1.0mm}\includegraphics[height=31mm]{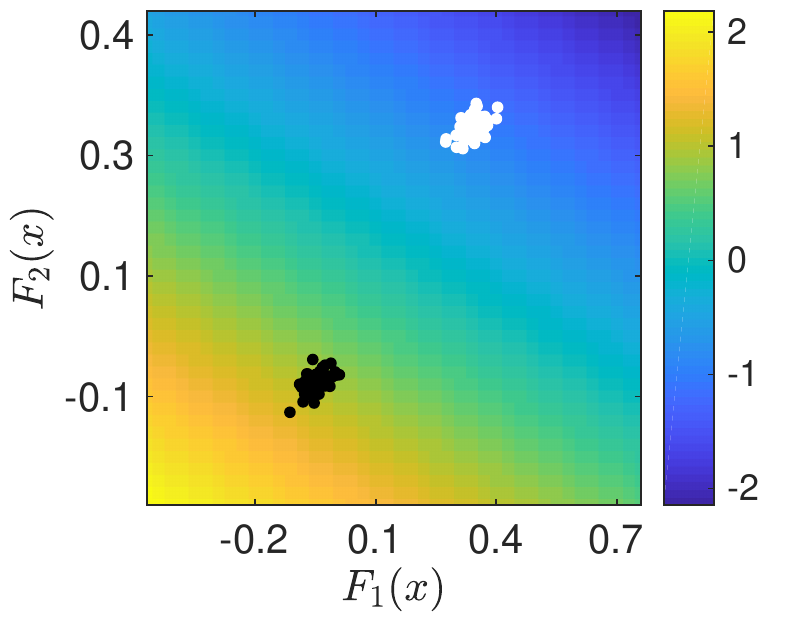} & 
    \hspace{-6.0mm}\includegraphics[height=31mm]{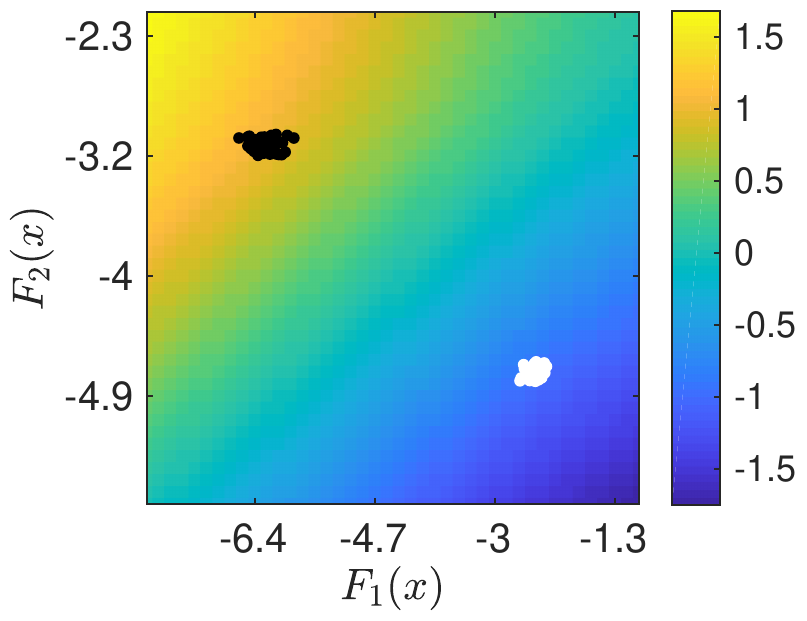} \\
    \multicolumn{2}{c}{(b) FOCA. Single- (left) and multi-layered classifiers (right).}
    \end{tabular}
    }
    \caption{Visualization of typical 2D features of two-class training data. 
    %whose input variables are generated from an 128D uniform distribution, 
    (a) A na\"ive joint optimization of a feature extractor and
    a classifier; 
    (b) FOCA (ours).
    Features in (b) form nearly point-like distributions per class,
    whereas those in (a) form more complex distributions.
    An L2 loss is minimized in each case.
    Black (white) dots indicate $+(-)1$-class data points, 
    and the colored maps indicate the classifiers' outputs,
    where in (b) ``averaged" outputs of 256 weak classifiers are shown.
    %The architecture of feature extractor is 128-128-128-128-2, and
    %that of the single (multi)-layered classifier is 2-1 (2-16-16-1).
    %Single (multi)-layered experiments use the sigmoid (leaky ReLU) activations.
    % MLP FOCA use L2-regularized weak classifiers
    }
    \label{fig_toy}
\end{figure}
% % plain, mlp => Toy20/ICML2019_plain_L2_mlp/fig21a.fig
% % FOCA, mlp => Toy20/ICML2019_FOCA_L2_mlp/fig21.fig
% % plain, analytical => Toy17/ICML2019_FOCA_analytical/fig21c.fig
% % FOCA, analytical => Toy20/ICML2019_plain_analytical/fig21b.fig

When specific signal patterns are repeatedly delivered 
by hidden neurons in a neural network
during training, 
the network parameters are updated in a strongly tied way,
or \textit{co-adapted}, so that the network becomes
vulnerable against small input perturbations
{\small\cite{Hinton2012,Srivastava2014}}.
To discourage co-adaptation,
Hinton~\etal proposed a method called Dropout that
randomly deactivates neurons during training.
Properties of Dropout training have been intensively 
studied {\small\cite{Helmbold2015OnTI, Baldi2013, Gal2016, Wager2013, Ren2016, WardeFarley2013, Bengio2013}};
whereas there is a critique saying
it does not necessarily yield 
co-adaptation prevention ability {\small\cite{Helmbold2018}}.

Yosinski~\etal studied the degrees of 
inter-layer co-adaptation 
by examining test performance that
mid-layer features can yield
{\small\cite{Yosinski2014}}.
In part of their experiments, 
they split an end-to-end trained network into two blocks of layers,
initialized the second-block parameters with random numbers, and
trained the second block from scratch with the first-block parameters held fixed.
They found that there are often cases where
the secondary optimization degrades the test performance
compared to the preceding primary joint optimization,
despite that Dropout is adopted.
In these cases,
inter-layer co-adaptation 
(or \textit{fragile co-adaptation}, in their words)
happens between two blocks.
Potentially, there is a chance that the secondary optimization finds 
the same minimum achieved by the primary optimization;
however, in reality the chance rate is usually not high.
%especially when one splits the block indifferently.
Excessively complex feature distribution, 
like the ones shown in Fig.~\ref{fig_toy}~(a),
would be a major factor that induces inter-layer co-adaptation.
Yosinski~\etal also showed that inter-layer co-adaptation tends to cause negative effects 
in the cross-domain transfer.
%When two successive layers are in co-adaptation,
%the intermediate feature vector performs poorly both in-domain and cross-domain.
%One of the concerns about a joint optimization of feature extractor and
%classifier is that the feature extractor co-adapts to the classifier 
%in a way that the feature-extraction ability degrades.

Is there a way to fundamentally
avoid inter-layer co-adaptation?
Based on a thought that na\"ive joint optimization of a feature extractor
and a classifier would result in unwanted co-adaptation
between them,
we seek a more fundamental approach to break the adhesion,
rather than
searching best-performing feature layers empirically {\small\cite{Yosinski2014,Kobayashi2017}}.
The questions we try to answer in this work are:
a) \textit{Is it possible to train a feature extractor 
without inter-layer co-adaptation to a particular classifier?}
b) \textit{After such training, what kind of characteristics 
along with robustness against unwanted inter-layer co-adaptation
does the feature extractor acquire?}

Regarding the first problem,
we introduce a particular feature-extractor optimization method 
called
Feature-extractor Optimization through Classifier 
Anonymization (FOCA) in Section~\ref{two}.
FOCA is designed so that the feature extractor 
does not explicitly co-adapt to a particular classifier.
Instead, it uses randomly generated, weak classifiers
during the feature-extractor training.
FOCA belongs to a family of network randomization methods
{\small\cite{Srivastava2014, dropconnect, stochasticpooling, stochasticdepth, swapout}},
but is different from others in terms that
FOCA does not employ a joint optimization of
a feature extractor and a classifier.
The classifier part is \textit{anonymized}
by marginalizing independently generated, weak classifiers;
in this way explicit co-adaptation to a particular classifier is avoided.
%In a joint optimization case,
%a feature extractor is updated so as to reduce the empirical loss
%for a given discriminative function defined by a classifier 
%at that stage.
%We replace the single classifier with many weak classifiers
%generated independently to \textit{anonymize} the classification %part.
%Thus, the feature extractor is connected to different classifier(s) %in every update;
%in this way explicit co-adaptation to a particular classifier is %avoided.
%In the practical algorithm, at each repetition,
%a (set of) weak classifier(s) is generated, 
%then the feature extractor is updated so as to make
%the weak classifier(s) slightly stronger.
%Intuitively speaking, repeating this process likely ends up with
%a large margin between features belonging to different classes.

Regarding the second problem,
we obtained intriguing 
mathematical proposition (Section~\ref{three}) and 
experimental evidences 
about simplicity of FOCA feature distributions.
Let us suppose class-$c$ features form a point-like distribution
in a class-separable fashion.
In that case 
a strong classifier for a partial dataset must be
also strong for the entire dataset.
This characteristics is largely confirmed for the FOCA features 
(Section~\ref{four-one}).
The distance between large-dataset solution and small-dataset solution
in the classifier parameter space 
is indeed very small when FOCA is adopted
(Section~\ref{four-two}).
Low-dimensional analyses of the FOCA features exhibit 
nearly point-like distributions
(Section~\ref{four-three}).

\begin{comment}
This study focuses on in-domain optimizations, by which 
we mean the primary and secondary optimizations use the same dataset.
%Cross-domain transfer learning is outside the scope.
\end{comment}

%%%%%%%%%%%%%%%%%%%%%%%%%%%%%%%%%%%%%%%%%%%%%%%%%%%%%%%%%%%%%%%%%%%%%%%%%%%%%%%%
\section{Optimization Method to Break Inter-Layer Co-Adaptation}
\label{two}

In this section, 
we introduce FOCA that aims at training a feature extractor
without inter-layer co-adaptation to a particular classifier.
We first go over the basic joint optimization method, then
introduce FOCA.

%%%%%%%%%%%%%%%%%%%%%%%%%%%%%%%%%%%%%%%%%%%%%%%%%%%%%%%%%%%%%%%%%%%%%%%%%%%%%%%%
\subsection{Joint Optimization: a Review}
\label{two-one}

Let $(x, t)$ be a pair of a $d_I$-dimensional input data and 
the corresponding $d_O$-dimensional target data, respectively.
The training dataset $\mathcal{D}$ contains $n_\mathcal{D}$ such pairs.
Feature extractor 
$F_\phi:\mathbb{R}^{d_I}\rightarrow\mathbb{R}^{d_F}$
transforms an input to a $d_F$-dimensional feature with parameter set $\phi$, and
classifier $C_\theta:\mathbb{R}^{d_F}\rightarrow\mathbb{R}^{d_O}$
transforms a feature to a $d_O$-dimensional output vector with parameter set $\theta$.
A joint optimization problem is given as
\begin{equation}
    \left(\phi^\star, \theta^\star \right) = \argmin_{\phi, \theta} \frac{1}{n_\mathcal{D}}
    \sum_{(x,t)\in\mathcal{D}} L\left(C_\theta(F_\phi(x)), t\right),
    \label{eq_joint}
\end{equation}
where $L(\cdot,t):\mathbb{R}^{d_O}\rightarrow\mathbb{R}$ defines 
the sample-wise loss between
the network output and the target.

%An orthodox SGD update rule is 
%\begin{align}
%    \omega_j &\leftarrow \omega_j + \delta\omega_j, \nonumber \\
%    \delta\omega_j &\propto - \sum_{(x,t)\in\mathcal{B}} \nonumber
%    \partial L / \partial\omega_j,
%\end{align}
%where $\omega$ is either $\phi$ or $\theta$, 
%and $\mathcal{B}$ is a fixed-sized, random subset of the dataset $\mathcal{D}$.
When SGD training is na\"ively applied,
at each step the classifier is updated so as to
become more discriminative for the presented features,
no matter how complex the feature distribution is.
The feature extractor, on the other hand, is updated 
so that the classifier at that moment becomes stronger, 
no matter how complex the decision boundary is.
The toy example in Fig.~\ref{fig_toy}~(a) demonstrates such a case, where
training results in excessively complex feature distribution.

%%%%%%%%%%%%%%%%%%%%%%%%%%%%%%%%%%%%%%%%%%%%%%%%%%%%%%%%%%%%%%%%%%%%%%%%%%%%%%%%
\subsection{Feature-extractor Optimization through Classifier Anonymization (FOCA)}
\label{two-two}

Below, we introduce FOCA
for optimizing a feature extractor 
without explicitly co-adapting to a particular classifier.
The optimization problem is defined as
\begin{equation}
    \phi^\star = \argmin_\phi \frac{1}{n_\mathcal{D}} \sum_{(x,t)\in \mathcal{D}}
    % \sum_{\theta\in\Theta_\phi} 
    \mathbb{E}_{\theta\sim\Theta_\phi} 
    L \left(C_\theta(F_\phi(x)), t \right),
    \label{eq_phistar}
\end{equation}
where $\Theta_\phi$ represents a predefined distribution function of 
\textit{weak classifiers} for a given parameter set $\phi$, and
$\mathbb{E}_{\theta\sim\Theta_\phi}$ represents the expectation value
over $\theta\sim\Theta_\phi$.
The feature extractor is optimized with respect to a 
\textit{set} of weak classifiers that are independently sampled from $\Theta_\phi$ and thus
is not able to co-adapt to a particular classifier,
as long as $\Theta_\phi$ generates distinct weak classifiers.

The weakness of the discriminative power of 
$\theta\sim\Theta_\phi$
is essential in this formulation.
If $\theta\sim\Theta_\phi$ is designed to be too strong for $\mathcal{D}$,
its decision boundary likely becomes fairly complex 
during the training, and
the feature extractor would update itself to better fit the
complex decision boundary,
resulting in a vicious cycle.
On the other hand, if $\theta\sim\Theta_\phi$ is too weak or even adversarial,
the optimization process would not converge.

The marginalization over weak classifiers likely prevents the feature distribution
from becoming excessively complex.
Even at the end of the optimization,
there is generally a large number of distinct weak classifiers, and
the feature extractor is optimized with respect to the \textit{ensemble} of these weak classifiers.
Although some of the weak classifiers may have excessively complex
decision boundaries, marginalization over the classifier ensemble 
likely smoothens those out.
This likely yields a relatively simple decision boundary 
and reasonably strong classification power,
as is essential in the classical Classifier Bagging algorithms {\small\cite{Breiman1996}} and 
other ensemble learning algorithms
{\small\cite{Hara2017, Feng2016}}.
Therefore,
the form of the feature distribution likely becomes simple
as far as the feature-extractor's description ability allows.

There is some room in defining $\Theta_\phi$, and
here we introduce a particular definition.
Let 
\begin{equation}
    \Theta_\phi = \mathcal{U}(\{\theta_{\phi, b} ; b = b_1, b_2, \cdots\}), 
    \label{eq_Theta}
\end{equation}
where $\mathcal{U}(s)$ is a discrete uniform distribution function 
for all elements in set $s$, and $\theta_{\phi, b}$ is a solution that minimizes
a batch-wise loss function with a norm regularization,
\begin{equation}
    \theta_{\phi, b} = \argmin_\theta \frac{1}{n_b} \sum_{(x,t)\in b} 
    L \left(C_\theta(F_\phi(x)), t \right) + \lambda \| \theta \|_2^2.
    \label{eq_theta}
\end{equation}
Here, batch $b$ comprises $n_b$ training samples that cover all classes,
and $\lambda > 0$.
We further assume that the classifier parameters
are initialized with random numbers prior to optimizing;
therefore, there is almost no chance of having continuity 
between $\theta_{\phi, b}$ and $\theta_{\phi+\delta\phi, b}$ 
for $\|\delta\phi\|\ll 1$.
%no matter how small the scale of $\delta\phi$ is.

A solution $\theta_{\phi,b}$, a strong classifier for the batch $b$, 
is not generally strong for the entire dataset $\mathcal{D}$
for given $\phi$
because it does not ``see" training samples other than the ones in $b$.
However, there is no guarantee that
$\theta_{\phi,b}$ is always a weak classifier to $\mathcal{D}$
in a classical sense; that is, 
a weak classifier performs only slightly better than random guesses.
Indeed, $\theta_{\phi,b}$ can even work adversarially to $\mathcal{D}$, 
meaning its accuracy is below the chance rate.
%However, when one takes an ensemble of plural classifiers generated by %Eq.~(\ref{eq_theta})
%with different $b$'s, it likely 
%acquires a reasonably strong classification ability.
But, for brevity, we simply call $\theta_{\phi, b}$ 
a ``weak classifier" in this work.

%The batch size $n_b$ is an important hyper-parameter as this controls the weakness.
%If $n_b$ is too large, $\theta_{\phi,b}$ likely becomes too strong for $\mathcal{D}$.
%If $n_b$ is too small, $\theta_{\phi,b}$ may become too weak or even adversarial for $\mathcal{D}$.

%\textcolor{red}{
%Regarding the number of weak classifiers used in a single update of $\phi$,
%it is not possible to prepare
%a complete set of possible weak classifiers because of
%the huge number of distinct batches of the same size.
%We must adopt an approximation method instead, such as 
%generating a single weak classifier
%per feature-extractor update.
%This works from our experience.}

The norm regularization term in Eq.~(\ref{eq_theta}) helps to avoid blowups 
during the feature-extractor training.
The scale of $\theta_{\phi,b}$ can be very large without the regularizer 
when two feature vectors in $b$ stand close to each other.
In such a case instability likely occurs.

%\textcolor{red}{
%It is worth mentioning that there is 
%another way of generating a reasonably weak classifier:
%to take a batch (which can be large) and then to optimize the batch-wise loss 
%in an incomplete fashion by stopping after a relatively small number of iterations, 
%say 20 times.
%This works fine, though
%the definition of $\theta_{\phi,b}$ becomes mathematically less clear.}

After a feature extractor is obtained by Eq.~(\ref{eq_phistar}),
the following secondary optimization using the entire dataset
provides a final, single classifier.
\begin{equation}
    \theta^\star = \argmin_\theta \frac{1}{n_\mathcal{D}}
    \sum_{(x,t)\in\mathcal{D}} 
    L \left(C_\theta(F_{\phi^\star}(x)), t \right).
    \label{eq_thetastar}
\end{equation}
Here, the classifier is trained with \textit{fixed} features.
Note that the classifier architecture
in this secondary optimization can differ from the one used in the primary 
optimization.% of the feature extractor.

Our method and meta-learning  
{\small\cite{Finn2017}} share a following similarity, 
despite that the goals are different 
(co-adaptation prevention vs. transferable multi-task learing).
Our feature extractor acts like task-generic base network, and
our classifiers act like taskwise fine-tuned models.

\textbf{Approximate minimization.}
Regarding the number of weak classifiers used in a single update of $\phi$,
it is impossible to prepare
a complete set of possible weak classifiers due to
the huge number of distinct batches of the same size.
We must adopt approximation instead.
Algorithm~\ref{alg1}~\footnote{In the pseudocode,
$\mathrm{randi}(i,j)$ 
returns a $j$-dimensional vector with each element being 
a random variable $\sim\mathcal{U}(\{1,2,\cdots,i \})$.}
gives an approximate solution of $\phi^\star$ 
%in Eq.~(\ref{eq_phistar})
%This algorithm provides an update rule for $\phi$,
%where a weak classifier is generated per $\phi$-update using a 
%randomly chosen batch of size $C$, which is the number of classes.
%This update rule yields only an approximate solution 
in two senses:
1) a single weak classifier is sampled from $\Theta_\phi$ per $\phi$-update
instead of taking a marginalization over $\Theta_\phi$, and
2) $\Theta_\phi$ is held fixed in the computation of gradients with respect to $\phi$.
%This strategy works in practice from our experience as we see in Section~\ref{three}.

\begin{algorithm}[h]
\caption{Approximate minimization in Eq.~(\ref{eq_phistar})}
\label{alg1}
\algblock{Begin}{End}
\begin{algorithmic}[1]
{\small
    \Require{total number of iterations $T$;
        number of classes $C$;
        number of class-$c$ samples $n_c {\footnotesize(c=1,\cdots,C)}$;
        number of samples per class for $\theta$-update $k$;
        total number of samples $n_\mathcal{D}$;
        minibatch size for $\phi$-update $m$;
        learning rate $\eta$}
    \Begin
    \State Initialize $\phi$ by random variables.
    \For{$t=1:T$}%\Comment{We have the answer if r is 0}
        \State $I_c = [ \mathrm{randi}(n_1,k), \cdots, \mathrm{randi}(n_C,k)]$
        %\Comment{Picks $k$ samples from each class.}
        \State $\theta = \argmin_{\theta'} \sum_{i\in I_c} L \left(C_{\theta'}(F_{\phi}(x_i)), t_i \right)     +\lambda \Vert \theta' \Vert_2^2$
        \State $I_f = \mathrm{randi}(n_\mathcal{D}, m)$
        %\Comment{Picks $m$ samples from the dataset.}
        \State $\phi \leftarrow \phi - \frac{\eta}{m} \sum_{i\in I_f} \partial     L\left(C_{\theta}(F_{\phi}(x_i)), t_i \right) / \partial \phi$
    \EndFor
    \End
    \Ensure{feature-extractor parameters $\phi^\star=\phi$}
}
\end{algorithmic}
\end{algorithm}

It is worth mentioning that there is 
another way of generating a reasonably weak classifier:
to take a batch (which can be large) and then to optimize the batch-wise loss 
in an incomplete fashion by stopping after a relatively small number of iterations, 
say 20 times.
This works fine, though
the definition of $\theta_{\phi,b}$ becomes mathematically less clear.

\section{Mathematical Property}
\label{three}

We now show a proposition about the simplicity of FOCA feature distributions.
It will be proven that under some special conditions
any two samples have exactly the same features when target classes are the same, 
but have different features when target classes are different.
Let us first introduce a lemma about implicit optimality for individual features,
and then put forth the proposition.

\begin{lemma}
    Suppose that a multi-layered feature extractor with two restrictions is used:
    
    1) The activation function $a$ satisfies 
    \begin{equation}
        a: \mathbb{R} \rightarrow \mathbb{R}^+, ~~\frac{\partial a(z)}{\partial z} \ne 0.
        \label{eq_activation}
    \end{equation}
    2) The last layer is fully-connected.
    
    If $\phi^\star$ simultaneously minimizes sample-wise losses
    $L(C_\theta(F_\phi(x)), t)$ for all $(x,t)\in\mathcal{D}$,
    then,
    \begin{equation}
        \frac{\partial {C_\theta}}{\partial F_{\phi^\star}}
        \frac{\partial L(C_\theta(F_{\phi^\star}(x)),t)}{\partial {C_\theta}}
        = 0, ~\forall (x,t) \in \mathcal{D}.
        \label{eq_feature_optimality}
    \end{equation}
% Ishikawa: Above equation should be interpreted as below. Expressions are little bit ambiguous (someone can see it something like functional derivatives for example). Adding a footnote or something for this explanation seems informative. (Especially, one can understand the parenthesized comment below clearly with it.)
%    \begin{equation}
%        \frac{\partial C_\theta(f)}{\partial f} \bigg|_{f=F_{\phi^*}(x)} \frac{\partial %L(c, t)}{\partial c} \bigg|_{c=C_\theta (F_{\phi^{*}} (x))} = 0, \forall (x,t) %\in \mathcal{D}.
%    \end{equation}
    ($\frac{\partial {C_\theta}}{\partial F_{\phi^\star}}$
    is a short-hand notation for
    $\left.\frac{\partial C_\theta(f)}{\partial f} \right|_{f=F_{\phi^*}(x)}$.
    A summation symbol over $C_\theta$ indices is ignored in 
    Eq.~(\ref{eq_feature_optimality}).)
    \label{lemma}
\end{lemma}
\begin{proof}
Let $\phi_\ell$ be the parameter set of the last weight layer in the feature extractor,
and let $x_\ell$ be its input.
Then, the $i$-th element of the feature layer,
which is fully-connected from the previous layer,
is given as
$F_\phi(i) = a(\sum_j\phi_\ell(i,j) x_\ell(j))$.
Let $z = \sum_j\phi_\ell(i,j) x_\ell(j)$.
Then, $\frac{\partial F_\phi(i)}{\partial \phi_\ell(i,j)} = \frac{\partial a(z)}{\partial z} x_\ell(j) \ne 0$,
since $\frac{\partial a(z)}{\partial {z}}\ne0$ and $x_\ell > 0$.
The inequality $\frac{\partial F_\phi(i)}{\partial \phi_\ell(i,j)}\ne 0$, 
the supposition $\left.\frac{\partial L}{\partial \phi_\ell(i,j)}\right|_{\phi=\phi^\star} = 0, 
\forall (x,t)\in\mathcal{D}$,
and the chain rule immediately leads Eq.~(\ref{eq_feature_optimality}).
\end{proof}

In the following discussion, we assume the conditions stated below hold.

\textbf{(C1)} A multi-layered feature extractor with two restriction is used:
1) The activation function satisfies Eq.~(\ref{eq_activation});
2) The last layer is fully-connected.

\textbf{(C2)} The target values are $t\in \{t_1, t_2\}$ for all samples.

\textbf{(C3)} A sample-wise loss function of the form
$\tilde{L}_{\phi,\theta}(x,t) = 
(C_\theta(F_\phi(x)) - t)^2$ 
is adopted.

\textbf{(C4)} A linear classifier 
$C_\theta(F_\phi(x)) = \bar{\theta}^\top F_\phi(x) + \theta^0$ is used.

\textbf{(C5)}
$\Theta_\phi = \mathcal{U}(\{\theta_{\phi, b} ; b = b_1, b_2, \cdots\})$, 
where
$\theta_b = \argmin_{\theta} \sum_{(x,t)\in b} 
\tilde{L}_{\phi,\theta}(x,t)
+ \frac{1}{2}\lambda \Vert \bar{\theta} \Vert_2^2$, 
and 
$b_1, b_2, \cdots$ are distinct batches, 
each of which comprises 
one sample from $t_1$ class and 
one sample from $t_2$ class.

%\textbf{(C5)} A classifier $\theta_b\sim\Theta_\phi$ satisfies
%$\theta_b = \argmin_{\theta} \sum_{(x,t)\in b} 
%\tilde{L}_{\phi,\theta}(x,t)
%+ \frac{1}{2}\lambda \Vert \bar{\theta} \Vert_2^2$, where
%$b$ represents a batch comprising one sample from each class.

\begin{proposition}
    Suppose that $\phi^\star$ simultaneously minimizes 
    the classifier-anonymized, sample-wise losses
    $\mathbb{E}_{\theta\sim{\Theta_\phi}} \tilde{L}_{\phi,\theta}(x,t)$
    in a class-separable fashion
    for all $(x,t)\in\mathcal{D}$.
    Then, samples from the same class share the same features; \ie,
    $F_{\phi^\star}(x) = F_{\phi^\star}(x'), \forall x,x'\in\mathcal{X}_c$, but
    samples from different classes do not; \ie,
    $F_{\phi^\star}(x) \neq F_{\phi^\star}(x'),~
    \forall x\in\mathcal{X}_c,~
    \forall x'\in\mathcal{X}_{c'\neq c}$.
    \label{proposition}
\end{proposition}
\begin{proof}
Lemma~\ref{lemma} 
about the implicit optimality of individual features 
with respect to sample-wise losses
yields,
\begin{equation}
    \mathbb{E}_\theta
    \frac{\partial C_\theta}{\partial F_{\phi^\star}}
    \frac{\partial \tilde{L}_{\phi^\star,\theta}(x,t)}{\partial C_\theta}
    = 0, ~\forall (x,t) \in \mathcal{D},
    \label{eq_BP}
\end{equation}
where $\mathbb{E}_\theta$ is a short-hand notation for
$\mathbb{E}_{\theta\sim\Theta_\phi}$.
By taking the partial derivatives in Eq.~(\ref{eq_BP}), one obtains,
\begin{equation}
    \left(\mathbb{E}_\theta
    \bar{\theta} \bar{\theta}^\top \right) F_{\phi^\star}(x) = 
    \mathbb{E}_\theta
    \bar{\theta} (t - \theta^0), ~\forall (x,t) \in \mathcal{D}.
    \label{eq_theta_and_F}
\end{equation}
The singular-value decomposition of the matrix consisting of column vectors 
$\bar{\theta}_b$ sampled from $\Theta_{\phi^\star}$ yields
$
    \left[ \bar{\theta}_{b_1}, \bar{\theta}_{b_2}, \cdots \right] = USV^\top ,
$
where diagonal elements of the positive diagonal matrix $S$
consists of the singular values aligned in decreasing order.
Then, $U^\top \bar{\theta}_b = [\bar{\theta}_b^{n \top}, 0,\cdots,0]^\top$,
where $\bar{\theta}_b^n$ is the non-singular components of $\bar{\theta}_b$.
Taking the non-singular part in Eq.~(\ref{eq_theta_and_F}),
\begin{equation}
    \left(\mathbb{E}_\theta
    \bar{\theta}^n \bar{\theta}^{n\top} \right) F_{\phi^\star}^n(x) = 
    \mathbb{E}_\theta
    \bar{\theta}^n (t - \theta^0), ~\forall (x,t) \in \mathcal{D}.
    \label{eq_thetan_and_F}
\end{equation}
Here, 
$U^\top F_{\phi^\star}(x) = \left[ F_{\phi^\star}^{n\top}(x), 
F_{\phi^\star}^{s\top}(x)\right]^\top$,
where $F_{\phi^\star}^n(x)$ is the corresponding 
non-singular part, meaning 
$F_{\phi^\star}^n(x)$ and $\bar{\theta}^n$ sharing the same dimension.
%and $F_{\phi^\star}^s(x)$ is the singular part.
The matrix
$\mathbb{E}_\theta \bar{\theta}^n \bar{\theta}^{n\top}$ is obviously invertible;
therefore, Eq.~(\ref{eq_thetan_and_F}) can be solved for $F_{\phi^\star}^n(x)$.
It then tells us 
that (a): $F_{\phi^\star}^n(x)$ depends only on $t$ and $\theta_b$'s.
On the other hand, the minimum-norm solution $\bar{\theta}_b$ satisfies,
\begin{equation}
    [(f_1 - f_2)(f_1 - f_2)^\top + \lambda \mathbb{I}] \bar{\theta}_b =
    (t_1 - t_2)(f_1 - f_2),
    \label{eq_L2solution}
\end{equation}
where $\mathbb{I}$ is the identity matrix and
$f_{1(2)} = F_{\phi^\star}$ with the target $t=t_1(t_2)$ as a short-hand notation.
Taking the non-singular part in Eq.~(\ref{eq_L2solution}),
the minimum-norm solution $\bar{\theta}_b^n$ satisfies (b):
%\begin{align}
%    &\left( \sum F^n_{\phi^\star} F^{n\top}_{\phi^\star}
%    - \frac{1}{2}\sum F^n_{\phi^\star} \sum F^{n\top}_{\phi^\star}
%    + \frac{\lambda}{2} \mathbb{I}_{d_F} \right) \bar{\theta}_b^n = \nonumber \\
%    &\sum t F^n_{\phi^\star}
%    - \frac{1}{2} \sum t \sum F^{n}_{\phi^\star},
%    \label{eq_L2solution}
%\end{align}
%where $\sum=\sum_{(x,t)\in b}$ and $\mathbb{I}_d$ is the $d \times d$ identity matrix.
\begin{equation}
    [(f^n_1 - f^n_2)(f^n_1 - f^n_2)^\top + \lambda \mathbb{I}] \bar{\theta}_b^n =
    (t_1 - t_2)(f^n_1 - f^n_2),
    \label{eq_L2solution_n}
\end{equation}
where superscript $n$ denotes the non-singular part.
Given the definition that 
$[\bar{\theta}^n_{b_1}, \bar{\theta}^n_{b_2}, \cdots]$ 
is full-rank,
statements (a) and (b) do not contradict only if 
\begin{equation}
    \exists v \in \mathbb{R}, ~\bar{\theta}^n_b = v, ~\forall b.
    \label{eq_theta_const}
\end{equation}
%\begin{equation}
%    \mathrm{rank}([ \bar{\theta}_{b_1}^n, \bar{\theta}_{b_2}^n, \cdots ]) = 1.
%    \label{eq_rank}
%\end{equation}
%It means at least that
%$\bar{\theta}_b^n \in \mathbb{R}$ and $F^{n}_{\phi^\star} \in \mathbb{R}$.
It is,
$\bar{\theta}^n_b$ is one-dimensional and constant for all $b$.
Then, statement (a) yields,
$F_{\phi^\star}^n(x) = F_{\phi^\star}^n(x'), \forall x, x' \in \mathcal{X}_c$.
Note that
$F_{\phi^\star}^n(x) \ne F_{\phi^\star}^n(x'), \forall x \in \mathcal{X}_c, \forall x'\in \mathcal{X}_{c'\ne c}$; otherwise $\phi^\star$ is not a class-separable solution.
Because $\theta^0_b = 1/2 \sum_{(x,t)\in b} (t - \bar{\theta}^{n\top} F_{\phi^\star}^n)$, 
$\theta^0_b$ must be the same for all $b$.
Since $\bar{\theta}_b = U [\bar{\theta}_b^{n\top}, 0, \cdots, 0]^\top$,
$\bar{\theta}_b$ must be the same for all $b$ also.
The fact that the minimum-norm solutions are the same for all combinations of
$t=t_1$ and $t=t_2$ data points tells that 
$F_{\phi^\star}(x) = F_{\phi^\star}(x'), \forall x, x' \in \mathcal{X}_c$ and
$F_{\phi^\star}(x) \ne F_{\phi^\star}(x'), \forall x \in \mathcal{X}_c, \forall x'\in \mathcal{X}_{c'\ne c}$.
\end{proof}

According to this proposition,
if the feature extractor has an enough representation
ability under certain conditions, 
all the input data of class $c$ are projected to a single point in the feature space in a class-separable way.
The left side of Fig.~\ref{fig_toy}~(b) visualizes 2D features of the toy data
optimized by FOCA under the conditions (C1)-(C5)\footnote{Here $\theta_{\phi, b}$ is the \textit{analytical} solution.}.
Features of the same class are confined in a vicinity, 
the size of which is
much smaller than the distance between the class centroids.
It is intriguing to observe such a ``point-like" distribution property,
even though we do not explicitly impose a thing like maximization of the
between-class scatter with respect to the sum of within-class scatters.
Although we have not succeeded in proving for multi-layered classifier cases,
the toy experiment still exhibits the point-like distribution 
property
as well; see the right side of Fig.~\ref{fig_toy}~(b).
%if Proposition~\ref{proposition} still holds for multi-layer classifier case,
%it means that we have a way to implicitly optimize the feature extractor
%so as to simplify the feature distribution to avoid the unwanted co-adaptation.

%%%%%%%%%%%%%%%%%%%%%%%%%%%%%%%%%%%%%%%%%%%%%%%%%%%%%%%%%%%%%%%%%%%%%%%%%%%%%%%%
\section{Experiment}
\label{four}

Let us first state our motivations for a series of experiments.
We saw in Section~\ref{three}
that the FOCA features obey 
a point-like distribution per class under the special conditions.
The question we try to answer in this section is,
\textit{Do the FOCA features form 
a point-like distribution or some similar distribution
under more realistic conditions?}
Let us suppose that features form a point-like distribution.
Then, following secondary classifier optimization 
with the feature extractor held fixed
should yield a similar decision boundary
no matter what subset of the entire dataset it learns, 
as long as all classes are covered.
We indeed confirmed
that high test performances are achieved by FOCA,
even when the secondary optimization uses smallest possible partial datasets;
namely, only one data from each class
(see Section~\ref{four-one}).
When FOCA is used,
the secondary optimization leads the classifier parameter vector
to almost the same point 
regardless of the size of the partial dataset it uses
(see Section~\ref{four-two}).
Lastly, 
low-dimensional analyses revealed that
FOCA features projected onto a hypersphere 
form a nearly a point-like distribution
in a class-separable fashion
(see Section~\ref{four-three}).

\textbf{Datasets.}
We use the CIFAR-10 dataset, 
a 10-class image classification dataset 
having $5\times 10^4$ training samples, and 
the CIFAR-100 dataset, 
a 100-class image classification dataset
having the same number of samples {\small\cite{Krizhevsky09learningmultiple}}.
Both datasets have similar properties except for 
the number of classes (10:100) and 
the number of samples per class (5000:500).
The idea is to see how these differences affect
feature distribution properties.

\textbf{Methods.}
In each experiment,
FOCA is compared with other methods below.
\textbf{Plain}: a vanilla mini-batch SGD training.
\textbf{Noisy} {\small\cite{Noisy}}: 
the same training rule as in Plain,
except that a zero-mean random Gaussian noise is added to 
each of the classifier parameters during training.
\textbf{Dropout} {\small\cite{Hinton2012}}: adopted only to the classifier part.
\textbf{Batch Normalization} {\small\cite{Ioffe2015}}: adopted to the entire architecture.

Dropout is claimed to reduce
co-adaptation {\small\cite{Hinton2012}}, 
though there is a counterpoint to this view {\small\cite{Helmbold2018}}.
FOCA, Dropout and Noisy share the same characteristics in a sense that 
the classifier's descriminative power is weakened and
the classifier ensemble is implicitly taken during training.
Apart from FOCA, Dropout and Noisy employ joint optimization, and
we are interested to see how this affects 
the robustness against inter-layer co-adaptation.
Batch Normalization is included in comparison based on a thought that
the way it propagates signals from one layer to the other 
may have some functionality mitigating inter-layer co-adaptation.

\textbf{Architecture.}
The architecture that we use for the primary optimization 
in each CIFAR-10 experiment
is the one introduced in {\small\cite{lee2016generalizing}}, 
except that we replaced the last two layers 
by three fully-connected (FC) layers of the form:
4096(feature dim.)-$\beta$-$\beta$-10, where $\beta=1024$ for Dropout and $\beta=128$ otherwise.
The architecture for the secondary optimization is 4096-128-128-10 for all methods.
%as this comparatively worked well in the validation for Plain.
The architecture for the CIFAR-100 primary optimizations
is VGG-16 {\small\cite{Simonyan2015}}, 
except that the last three FC layers are replaced by
512(feature dim.)-$\beta$-$\beta$-100, where $\beta=512$ for Dropout and $\beta=128$ otherwise.
The architecture for the secondary optimization is 512-128-128-100 for all methods.

\textbf{Training details.}
SGD with momentum is used in each baseline experiment.
In each FOCA experiment, 
the feature-extractor part uses SGD with momentum, and
the classifier part uses gradient descent with momentum.
In each training, we tested a couple of different initial learning rates
and chose the best-performing one in the validation.
A manual learning rate scheduling is adopted;
the learning rate is dropped by a fixed factor 1-3 times.
The weak classifiers are randomly initialized each time by zero-mean Gaussian distribution with standard deviation $0.1$ for both CIFAR-10 and -100.
Cross entropy loss with softmax normalization 
and ReLU activation 
{\small\cite{ReLU}}
are used in every case.
No data augmentation is adopted.
The batch size $b$ used in the weak-classifier training is
100 for the CIFAR-10 and 1000 for the CIFAR-100 experiments.
The number of updates to generate $\theta$ is 32 for the CIFAR-10 and 64 for the CIFAR-100 experiments.
Max-norm regularization {\small\cite{Srivastava2014}} is used for the FOCA training, 
to stabilize the training.
We found that the FOCA training can be made even more stable when
updating the feature-extractor parameters
$u$ times for 
a given weak classifier parameters.
We used this trick with $u=8$ in the CIFAR-100 experiments.

\subsection{Test Performances of Classifiers Trained on Partial Datasets}
\label{four-one}

% alpha/P x-axis
% alpha=0 SD
% total dist -> table
\begin{figure}[t]
    {\small
    \begin{center}
    \begin{tabular}{cc}
    \hspace{-2mm}\includegraphics[height=40mm]{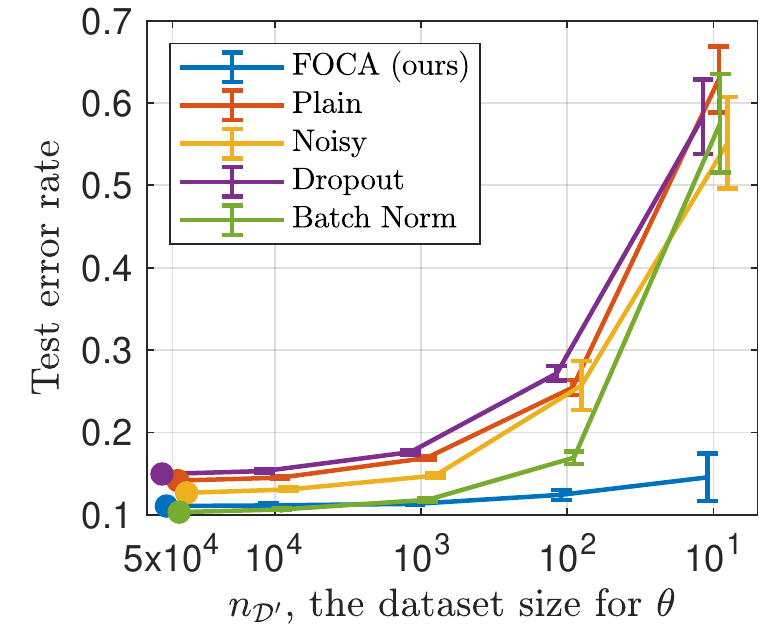} &
    \hspace{-3mm}\includegraphics[height=40mm]{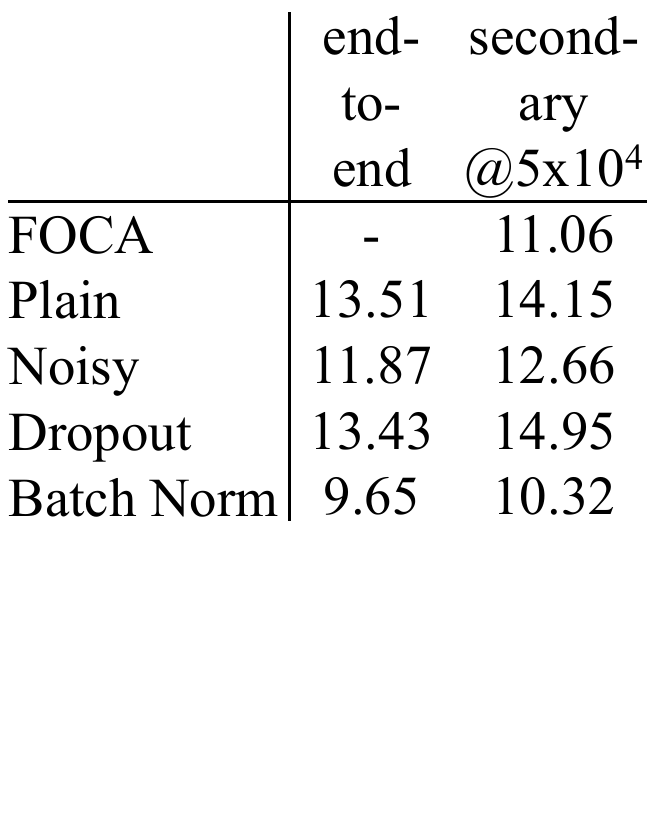}
    \end{tabular}
    (a) CIFAR-10
    \begin{tabular}{cc}
    \hspace{-2mm}\includegraphics[height=40mm]{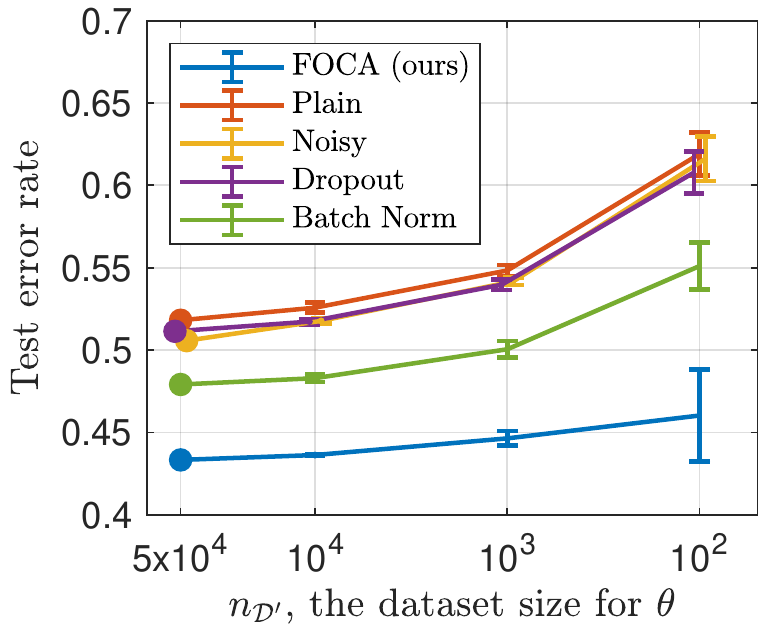} &
    \hspace{-3mm}\includegraphics[height=40mm]{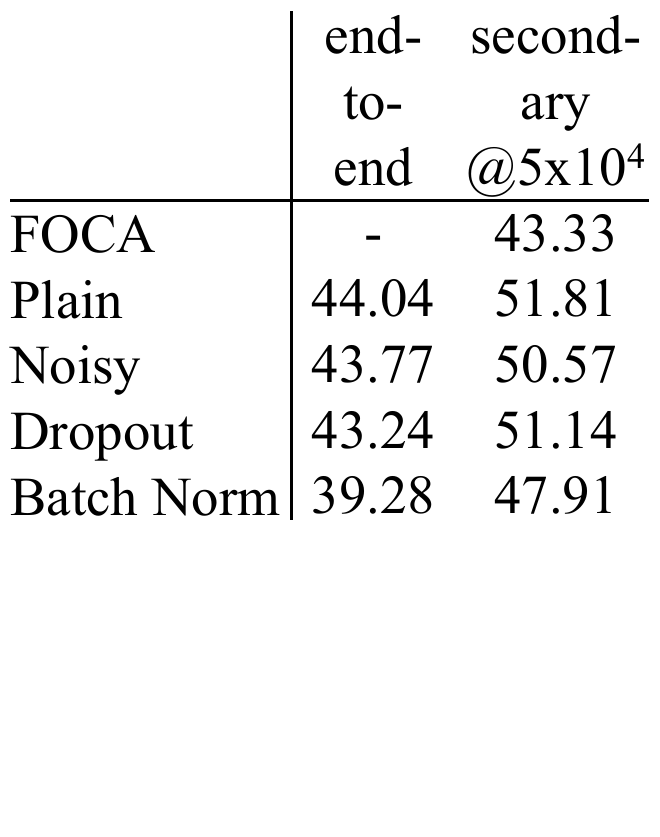}
    \end{tabular}
    (b) CIFAR-100
    \end{center}
    \caption{Test error rates of classifiers trained on partial datasets.
    For each method, partial datasets $\mathcal{D}'$
    are constructed $r(n_{\mathcal{D}'})$ times,
    where $r(5\times 10^4)=1$, $r(10^4)=5$, $r(10^3)=15$, $r(10^2)=50$, and
    $r(10)=150$. 
    The solid lines indicate the mean values and
    the error bars indicate $\pm1$ standard deviations of test error rates.
    The numbers in the right side indicate
    the test error rates (\%) of the end-to-end optimizations and 
    the secondary optimizations for $n_{\mathcal{D}'}=5\times 10^4$ 
    (also shown by the bullet points).
    }
    \label{fig_smalldataset}
    }
\end{figure}

\textbf{Motivation.}
We are interested in two aspects.
1) Is the test performance produced by the primary joint optimization 
reproducible by the secondary classifier optimization?
(FOCA is excluded here because it is not a joint optimization method.)
2) How does the test performance degrade when smaller datasets are learned 
in the secondary optimization?
Remember that a point-like distributed features are expected to demonstrate
little degradation.

\textbf{Experimental procedure.}
We trained feature extractors on the entire training dataset using 
FOCA and the other methods.
All methods except for FOCA employ the end-to-end learning scheme.
Then, we detach the feature extractors and the classifiers after learning.
Next, for all methods,
we fix the feature-extractor parameters and train
classifiers from scratch by the orthodox backpropagation
on \textit{reduced datasets} $\mathcal{D'}$ of size $n_{\mathcal{D}'}$.
For CIFAR-10 experiments, $n_{\mathcal{D}'}=5\times 10^4, 10^4, 10^3, 10^2, 10$
($10$ is the smallest possible dataset size).
For CIFAR-100 experiments, $n_{\mathcal{D}'}=5\times 10^4, 10^4, 10^3, 10^2$
($10^2$ is the smallest possible).

Figure~\ref{fig_smalldataset} now shows the test error rates vs. $n_{\mathcal{D}'}$.

\textbf{Cases of $n_{\mathcal{D}'}=n_\mathcal{D}$.}
For all end-to-end optimization methods tested here,
the test performances after the secondary optimizations at $n_{\mathcal{D}'}=5\times 10^4$ 
happen to be worse than corresponding end-to-end optimization results
(see the right side of Fig.~\ref{fig_smalldataset}).
This is one of the indications that inter-layer co-adaptation 
occurs to some degree in each of these joint optimization methods.

\textbf{Cases of $n_{\mathcal{D}'}<n_\mathcal{D}$.}
For CIFAR-10, when $n_{\mathcal{D}'} \le 10^3$,
FOCA outperforms the other methods.
To our surprise, the average test error rate of FOCA at $n_{\mathcal{D}'}=10$ is
14.45\%, 
which is indeed better than the test error rates of 
Dropout (14.95\%) at $n_{\mathcal{D}'}=5\times 10^4$ 
Plain (14.50\%) at $n_{\mathcal{D}'}=10^4$,
Noisy (14.75\%) at $n_{\mathcal{D}'}=10^3$, and
Bach Normalization (16.93\%) at $n_{\mathcal{D}'}=10^2$.
%This gives an evidence 
%that there is a case where
%a classifier that only trained features of 
%a very small dataset, consisting of one sample from each class,
%can obtain the generalization ability of a network that is end-to-end trained with
%a $5000$ times larger dataset.
For CIFAR-100, FOCA outperforms the other methods at any $n_{\mathcal{D}'}$.
The average error rate of FOCA at $n_{\mathcal{D}'}=10^2$ is 46.03\%,
which is indeed better than the test error rate of all 
the other methods at
$n_{\mathcal{D}'}=n_\mathcal{D}=5\times 10^4$.
We think that a high test performance of FOCA 
for $n_{\mathcal{D}'} \ll n_\mathcal{D}$
is one indication of relatively simple form of feature distribution.

%%%%%%%%%%%%%%%%%%%%%%%%%%%%%%%%%%%%%%%%%%%%%%%%%%%%%%%%%%%%%%%%%%%%%%%%%%%%%%%%
\subsection{Approximate Geodesic Distances between Solutions}
\label{four-two}

\begin{figure*}[t]
    \centering
    {\small
    \begin{tabular}{ccccc}
    \hspace{-2.0mm}\includegraphics[height=35.4mm]{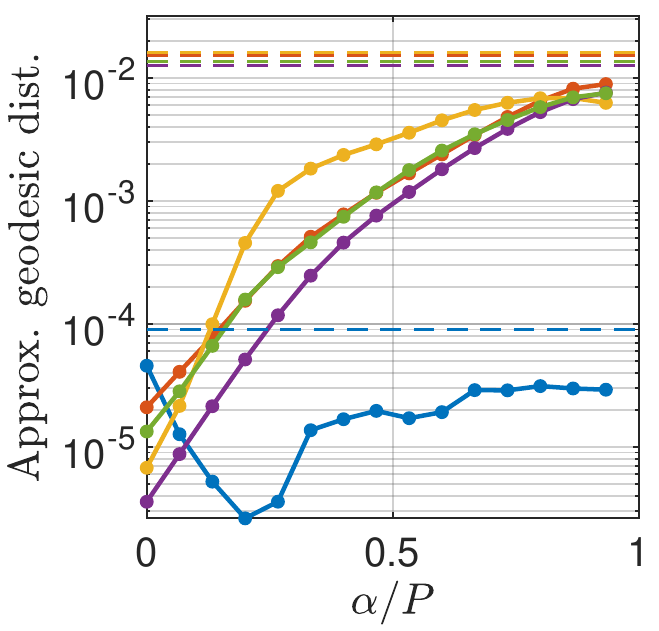} &
    \hspace{-3.6mm}\includegraphics[height=35.4mm]{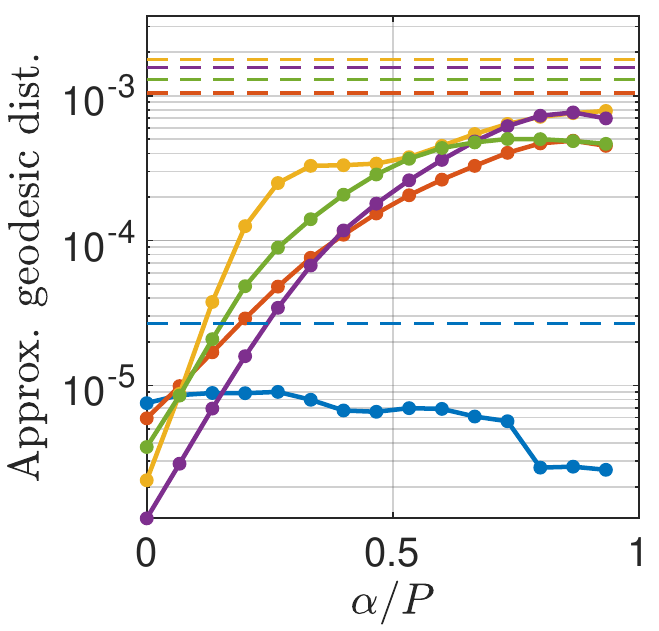} &
    \hspace{-3.6mm}\includegraphics[height=35.4mm]{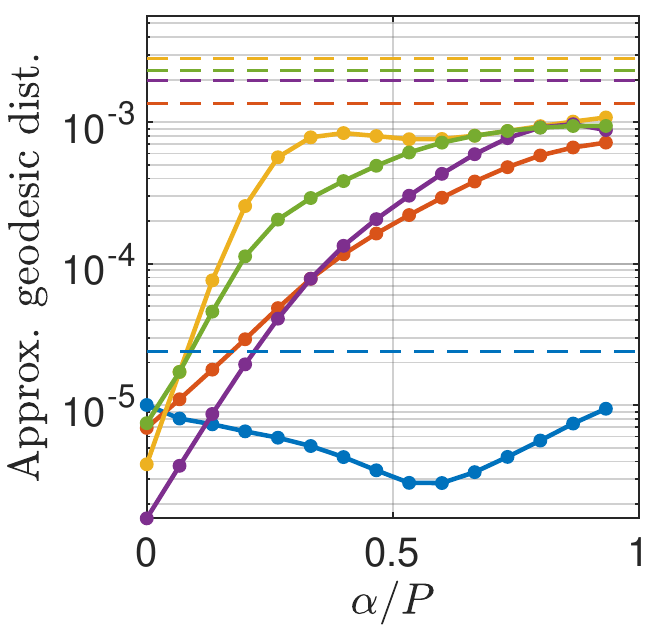} &
    \hspace{-3.6mm}\includegraphics[height=35.4mm]{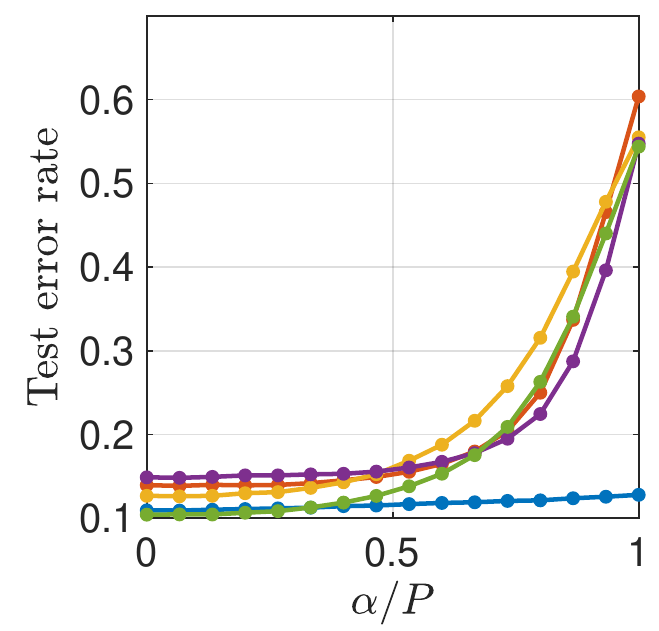} &
    \hspace{-5.0mm}\includegraphics[width=22mm]{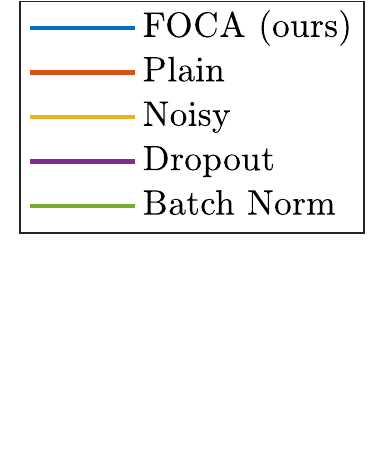} \\
    (a) CIFAR-10, 1st layer & (b) CIFAR-10, 2nd layer & 
    (c) CIFAR-10, 3rd layer & (d) CIFAR-10, error rate & \\
    \hspace{-2.0mm}\includegraphics[height=35.4mm]{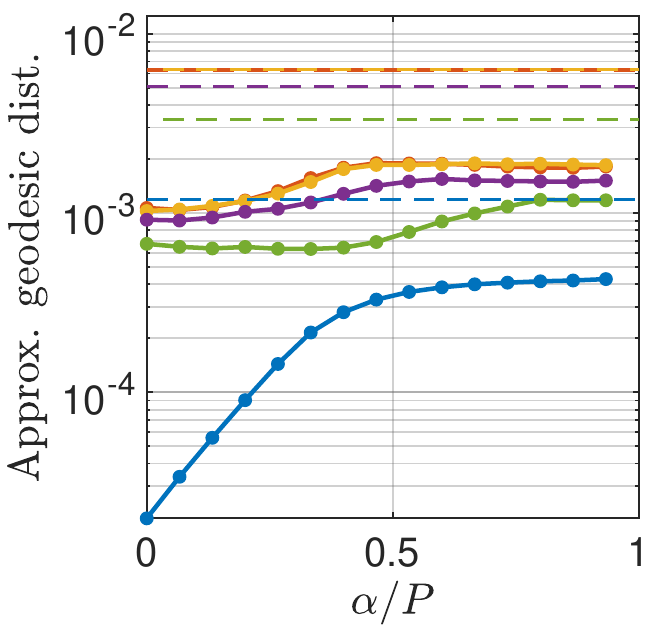} &
    \hspace{-3.6mm}\includegraphics[height=35.4mm]{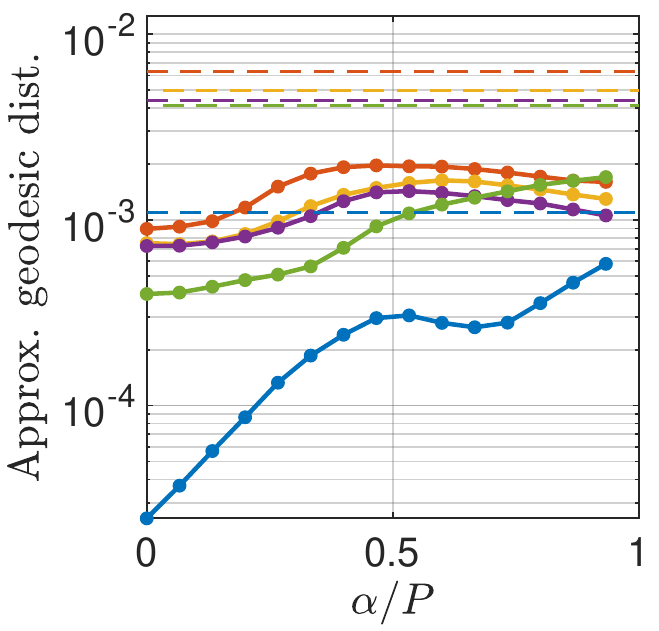} &
    \hspace{-3.6mm}\includegraphics[height=35.4mm]{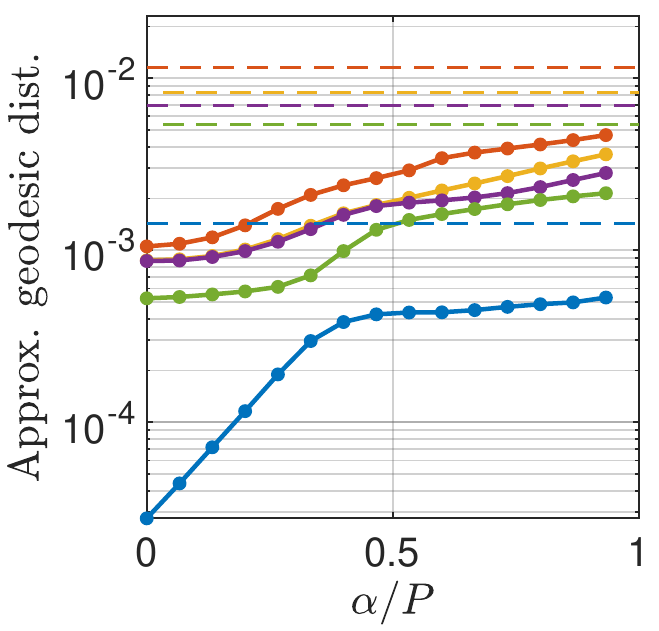} &
    \hspace{-3.6mm}\includegraphics[height=35.4mm]{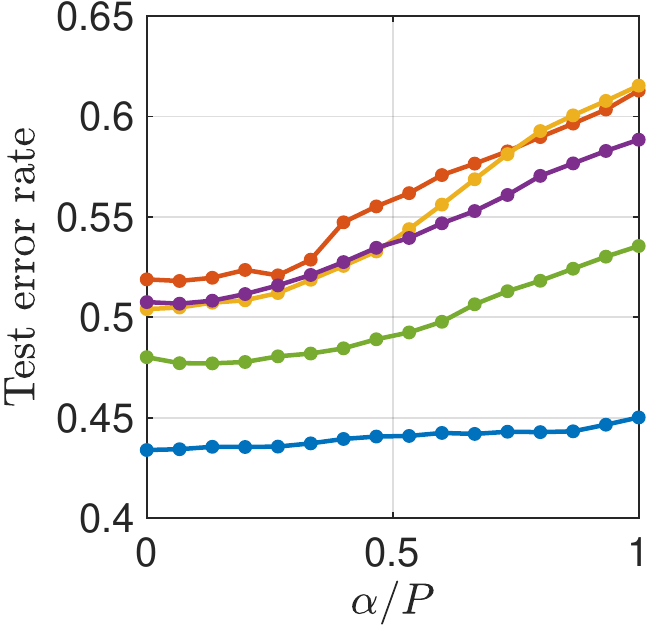} & \\
    %\hspace{-3.5mm}\includegraphics[width=22mm]{fig/GeodesicDist_lgd_crop.pdf} \\
    (e) CIFAR-100, 1st layer & (f) CIFAR-100, 2nd layer & 
    (g) CIFAR-100, 3rd layer & (h) CIFAR-100, error rate & 
    \end{tabular}
    }
    \caption{The approximate geodesic distances (a-c, e-g) 
    between $\theta^{LD}$ and $\theta^{SD}$
    and test error rates at $\theta^\alpha$ (d, h).
    In (a-c, e-g), the solid lines indicate segment-wise distances
    $d(\theta^\alpha,\theta^{\alpha+1}),\alpha\in\{0,\cdots,P-1\}$,
    and the dashed lines indicate total distances $d(\theta^{LD}, \theta^{SD})$.}
    \label{fig_geodesic}
\end{figure*}

\textbf{Motivation.}
We just saw that
classifiers trained with the largest possible dataset 
($n_{\mathcal{D}'}=n_{\mathcal{D}}$) and
classifiers trained with the smallest possible partial dataset 
($n_{\mathcal{D}'}=10$ for CIFAR-10, $n_{\mathcal{D}'}=100$ for CIFAR-100) 
exhibit similar test performances when the FOCA features are used.
Let us call the former the large-dataset solution $\theta^{LD}$
and the latter the small-dataset solution $\theta^{SD}$.
The above observation drove us to investigate 
distances between $\theta^{LD}$ and $\theta^{SD}$.
The distance should be small 
when the features form a point-like distribution per class,
and we expect FOCA has this characteristics.
We employ the \textit{approximate geodesic distance} here 
to take changes in the loss landscape into account.
If $\theta^{LD}$ and $\theta^{SD}$ are virtually the ``same" point with FOCA,
we further expect that the test error rate is almost unchanged
at any intermediate point between $\theta^{LD}$ and $\theta^{SD}$.
Since two neural networks having the same architecture and different parameters can
produce the same output for an arbitrary input {\small\cite{Watanabe2009}},
we initialized networks with the same set of random numbers
in the experiments conducted in this subsection.

\textbf{Experimental procedure.}
After optimizing a feature extractor on the full training dataset,
$\theta^{LD}$ is optimized with features of the full dataset of size 
$n_{\mathcal{D}'}=n_\mathcal{D}=5\times 10^4$, and 
$\theta^{SD}$ is optimized with features of a smallest possible partial dataset;
$n_{\mathcal{D}'}=10$ for CIFAR-10 and 
$n_{\mathcal{D}'}=100$ for CIFAR-100.
To quantify the separation between $\theta_{LD}$ and $\theta_{SD}$,
we partition the straight line connecting $\theta_{LD}$ and $\theta_{SD}$
into $P$ line segments 
of equal lengths in the parameter space; \ie,
\begin{equation}
    \theta^\alpha = \frac{\alpha \theta_{SD}+(P-\alpha)\theta_{LD}}{P},~
    \alpha = 0, 1, \cdots, P.
    \label{eq_15}
\end{equation}
We define in this article 
an approximate geodesic distance $d(\theta^{LD}, \theta^{SD})$ between 
$\theta^{LD}$ and $\theta^{SD}$
as 
\begin{equation}
    d(\theta^{LD}, \theta^{SD}) = \left[ \sum_{\alpha=0}^{P-1}
    d(\theta^\alpha, \theta^{\alpha+1})^2 \right]^\frac{1}{2}.
\end{equation}
Here, $d(\theta^\alpha, \theta^{\alpha+1})$ is the distance between 
$\theta^\alpha$ and $\theta^{\alpha+1}$
with respect to the Fisher information metric 
$\mathcal{I}^\alpha$ evaluated 
at $\theta^\alpha$,
\begin{align}
    &d(\theta^\alpha, \theta^{\alpha+1})^2 =
    (\theta^{\alpha+1} - \theta^\alpha)^\top \mathcal{I}^\alpha
    (\theta^{\alpha+1} - \theta^\alpha), \\
    &\mathcal{I}^\alpha = \mathbb{E}_{(x,t)\in\tilde{\mathcal{D}}}
    \left. \left(\frac{\partial L_{\phi^\star,\theta}(x,t)}{\partial \theta}\right) 
    \left(\frac{\partial L_{\phi^\star,\theta}(x,t)}{\partial \theta}\right)^\top
    \right|_{\theta=\theta^\alpha}, \nonumber
\end{align}
where $\tilde{\mathcal{D}}$ is either $\mathcal{D}$ or
a subset of $\mathcal{D}$ for ease of computation, and
$L_{\phi^\star, \theta}(x,t)$ is a short-hand notation of 
$L(C_\theta(F_{\phi^\star}(x)), t)$.
To compute a genuine geodesic distance, one needs to compute
the sum of Fisher-metric distances between pairs of 
infinitesimally separated points
along the curve that minimizes the squared sum.
This is computationally infeasible;
we instead approximate the curve by the straight line, as explained.
In the experiment, we let $\tilde{\mathcal{D}}$ be
a randomly chosen subset consisting of 5\% of the entire training samples.
We set $P=15$.
We evaluated $d(\theta^{LD}, \theta^{SD})$ layer by layer.

Figure~\ref{fig_geodesic} now shows the approximated geodesic distances 
and test error rates at $\theta^\alpha$.

\textbf{Approximate geodesic distances.}
For CIFAR-10,
FOCA exhibits some orders-of-magnitude, say 40-180 times, smaller
distances $d(\theta^{LD}, \theta^{SD})$ than the other methods.
For CIFAR-100,
FOCA exhibits 3-9 times smaller distances than the other methods.
We guess the reason why the differences are moderate for the CIFAR-100 cases
is that point-like distribution is harder to obtain for CIFAR-100.
%the number of classes is 10 times larger while the number of samples %is 10 times smaller 
%than CIFAR-10.
Nevertheless, FOCA exhibit smallest approximate geodesic distances in all cases, and
we think this is an implicit evidence that the distribution of the FOCA features 
is simple enough so that the discriminative function generated with $n_{\mathcal{D}'}=n_{\mathcal{D}}$ 
is virtually reproducible 
when $n_{\mathcal{D}'} \ll n_\mathcal{D}$.
%This is another indication that inter-layer co-adaptation is 
%either absent or at least inconspicuous.

%evidence that
%the excessively complex 
%feature distribution is absent
%when FOCA is used
%so that...

\textbf{Test error rates at $\theta^\alpha$.}
For FOCA, test error rates are almost constant at 
all $\theta^\alpha$.
Together with the small $d(\theta^{LD}, \theta^{SD})$,
two points $\theta^{LD}$ and $\theta^{SD}$ could be 
viewed as virtually the same point, when FOCA is used.
For other methods, 
test error rates increase more rapidly
toward $\theta^{SD}=\theta^{15}$.

%%%%%%%%%%%%%%%%%%%%%%%%%%%%%%%%%%%%%%%%%%%%%%%%%%%%%%%%%%%%%%%%%%%%%%%%%%%%%%%%
\subsection{Low-Dimensional Properties}
\label{four-three}

\begin{figure*}[tb]
%\if0 % % % % % % % % % % % % % % % % % % % % % % % % % % % % % % % % % % % % % % % % 
    \centering
    {\small
    \begin{tabular}{ccccc}
    \hspace{-2.0mm}\includegraphics[height=28.0mm]{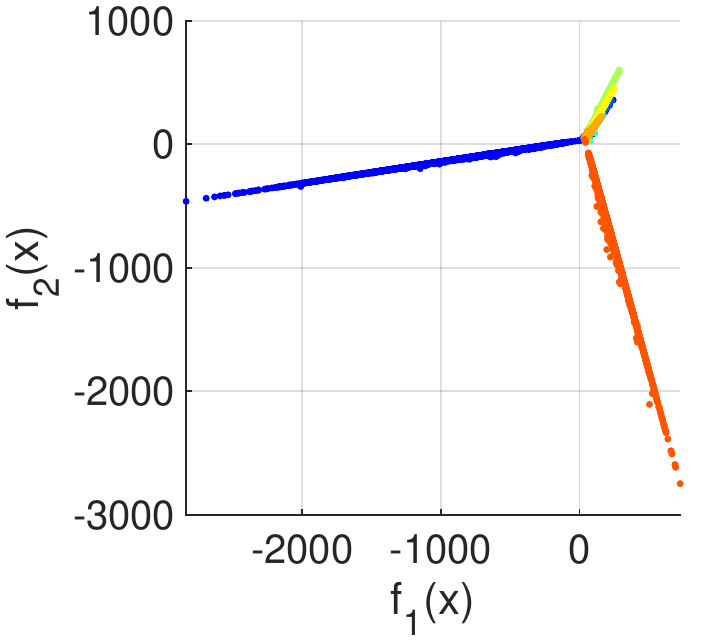} &
    \hspace{-4.0mm}\includegraphics[height=28.0mm]{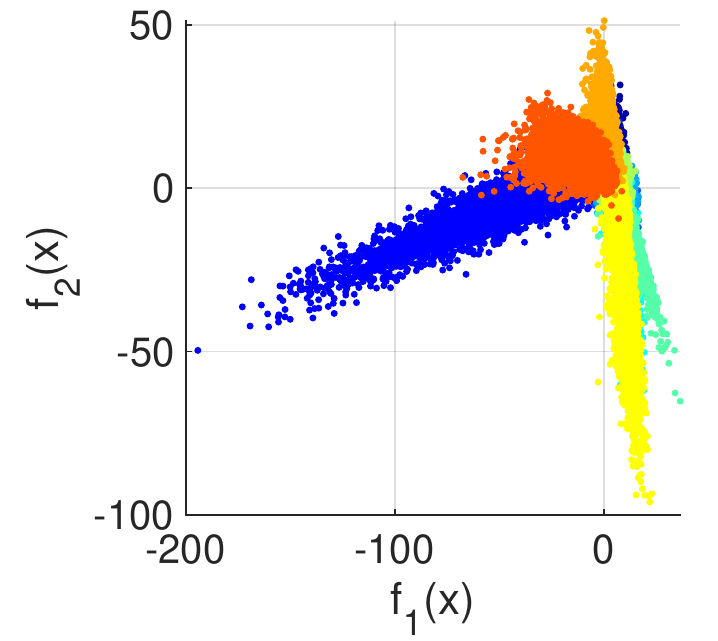} &
    \hspace{-4.0mm}\includegraphics[height=28.0mm]{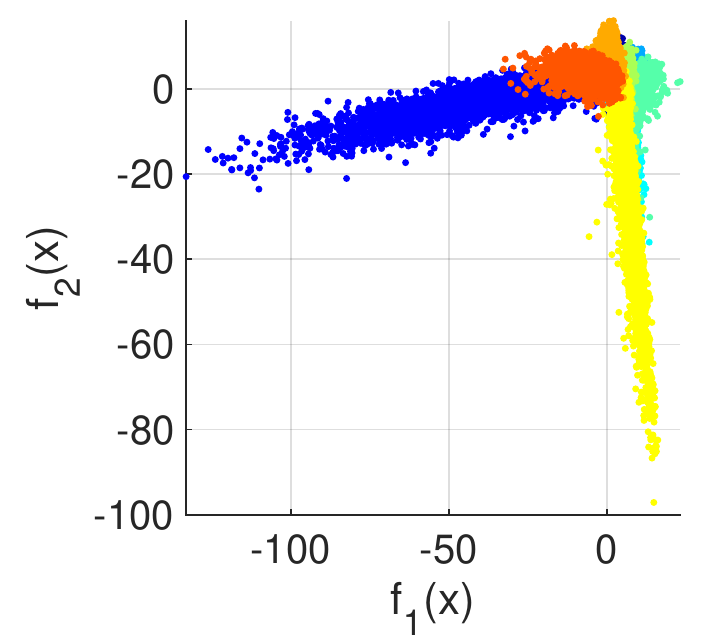} &
    \hspace{-4.0mm}\includegraphics[height=28.0mm]{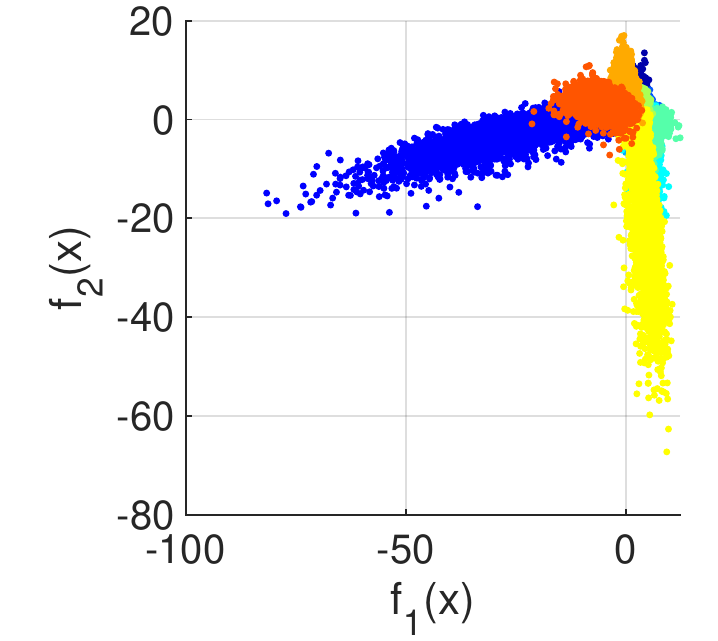} &
    \hspace{-4.0mm}\includegraphics[height=28.0mm]{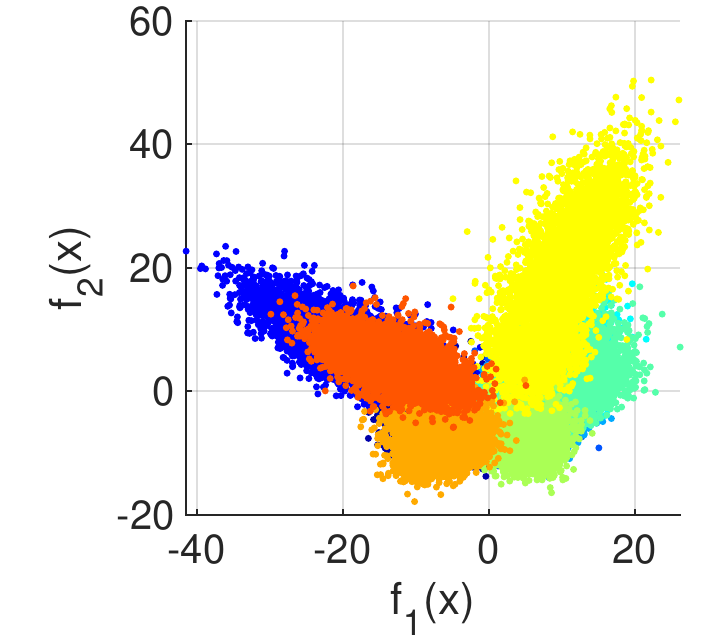} \\
    \multicolumn{5}{c}{(a) Training-data features projected by the PCA bases.} \\
    \hspace{-2.0mm}\includegraphics[height=28.0mm]{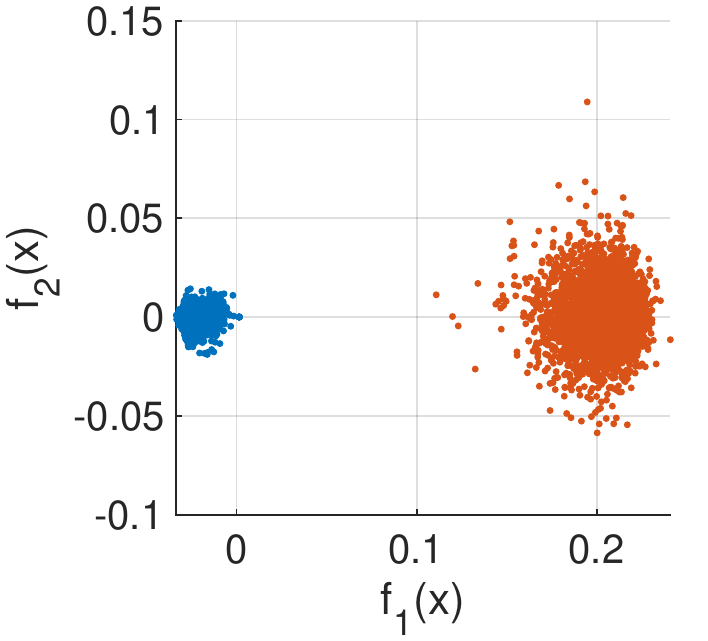} &
    \hspace{-4.0mm}\includegraphics[height=28.0mm]{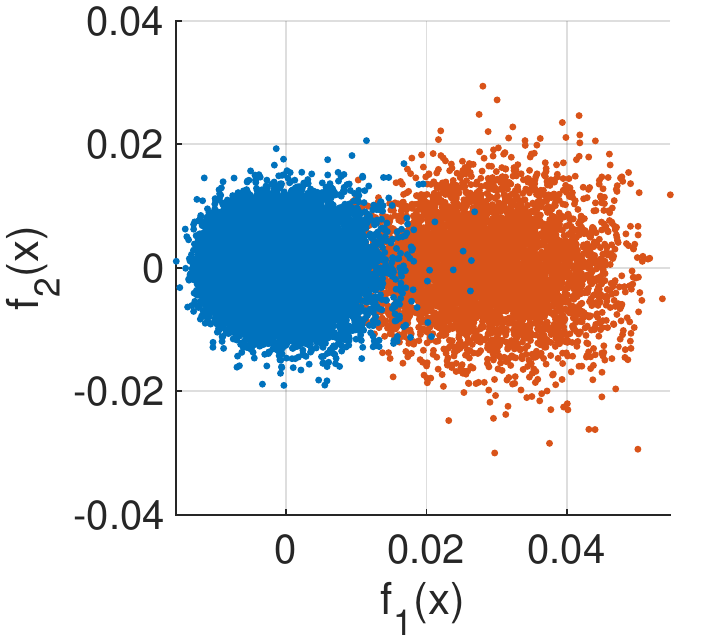} &
    \hspace{-4.0mm}\includegraphics[height=28.0mm]{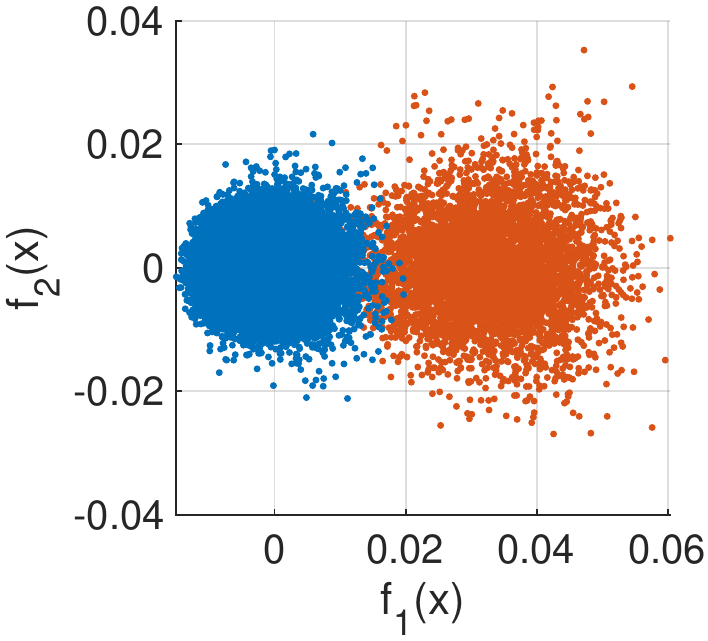} &
    \hspace{-4.0mm}\includegraphics[height=28.0mm]{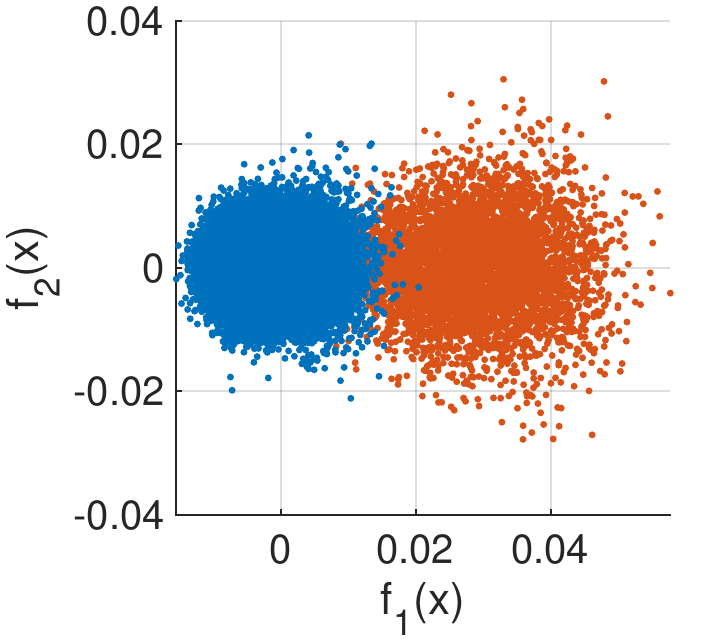} &
    \hspace{-4.0mm}\includegraphics[height=28.0mm]{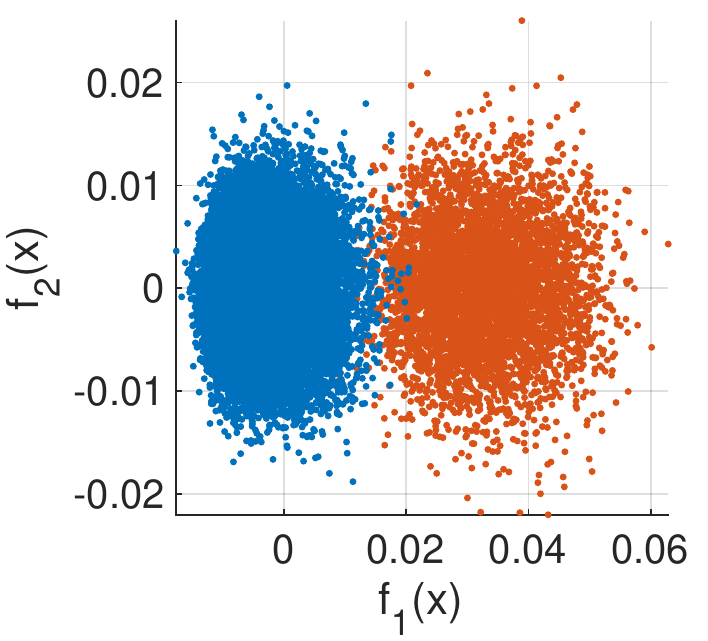} \\
    \multicolumn{5}{c}{(b) Training-data features first normalized 
    to unit lengths then projected by the LDA bases,} \\
    \multicolumn{5}{c}{constructed on the class-1, normalized features vs. 
    the rest of the normalized features.} \\
    \hspace{-2.0mm}\includegraphics[height=28.0mm]{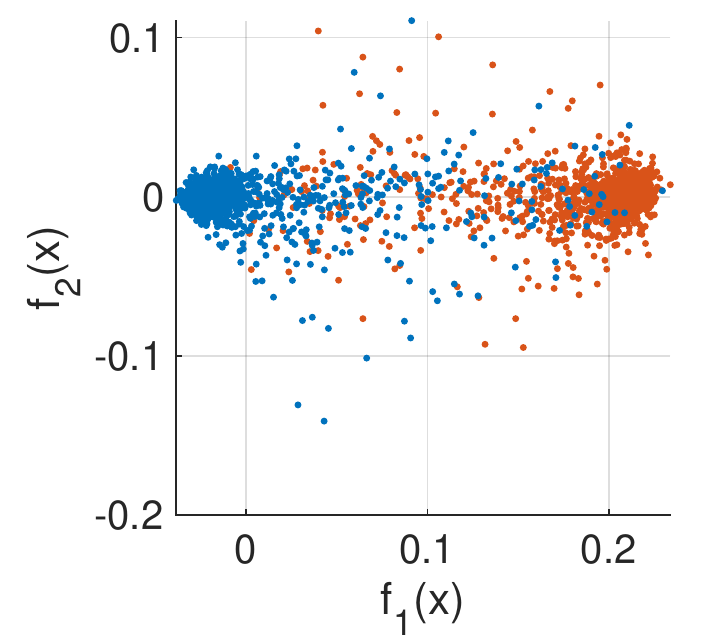} &
    \hspace{-4.0mm}\includegraphics[height=28.0mm]{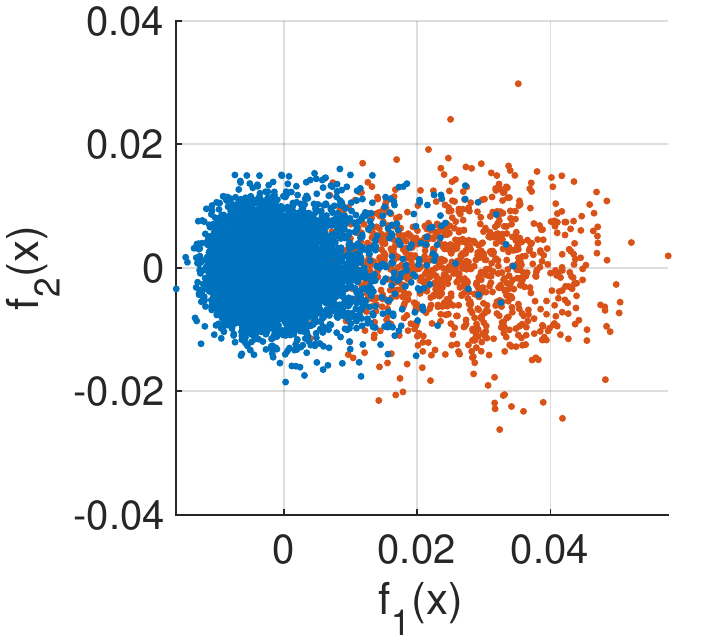} &
    \hspace{-4.0mm}\includegraphics[height=28.0mm]{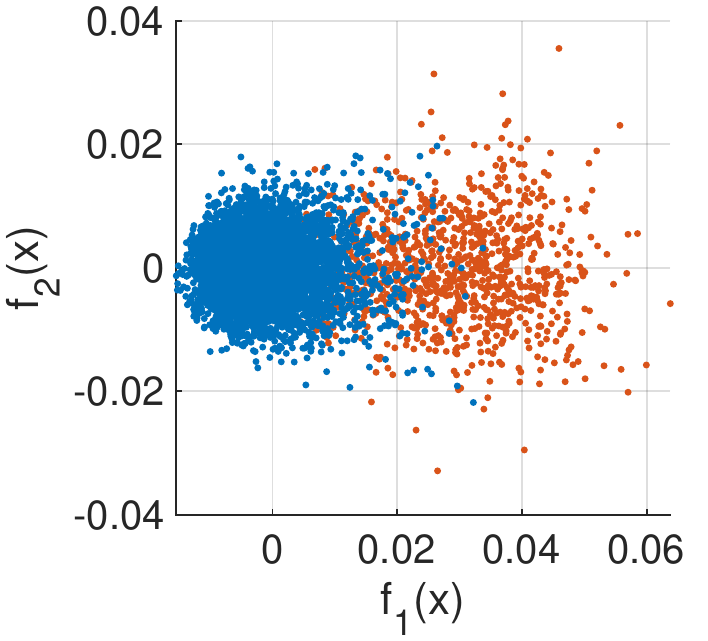} &
    \hspace{-4.0mm}\includegraphics[height=28.0mm]{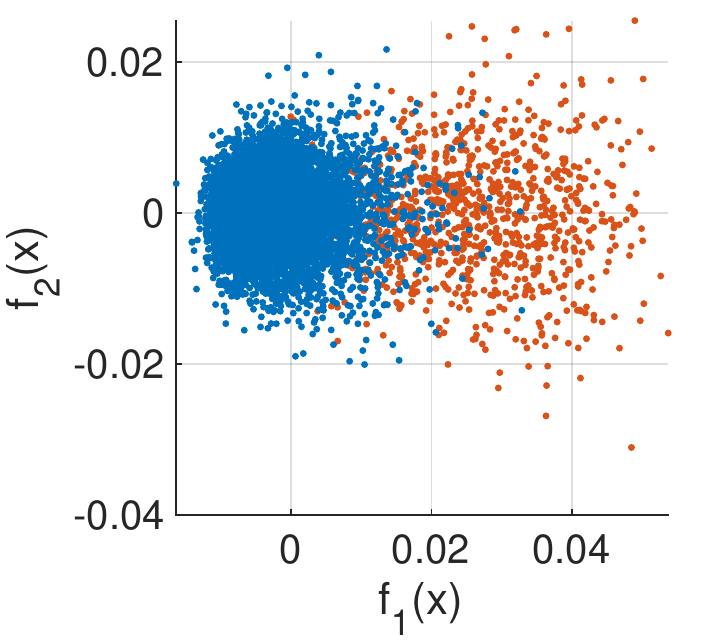} &
    \hspace{-4.0mm}\includegraphics[height=28.0mm]{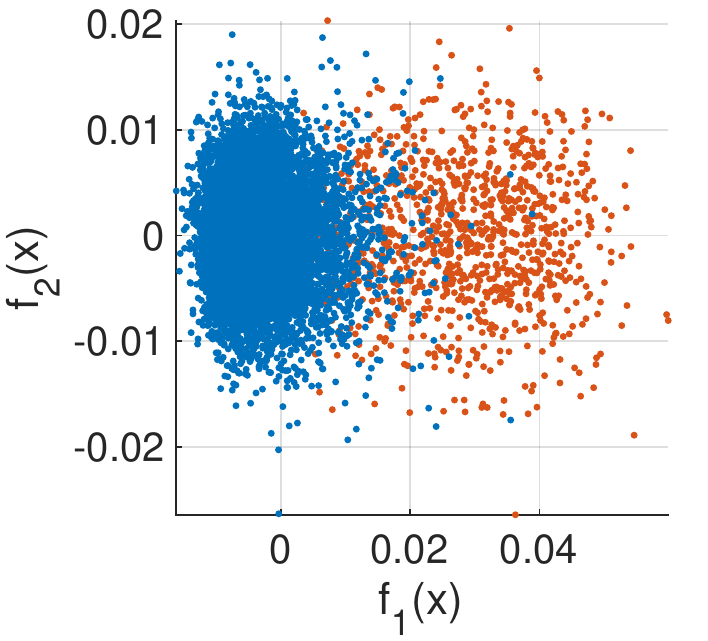} \\
    \multicolumn{5}{c}{(c) Test-data features first normalized to unit lengths 
    then projected by the same LDA bases used in (b).}
    \end{tabular}
    }
%\fi % % % % % % % % % % % % % % % % % % % % % % % % % % % % % % % % % % % % % % % % 
\caption{Two-dimensional visualization of 4096-dimensional, CIFAR-10 features.
Methods are (from the left): FOCA (ours), Plain, Noisy, Dropout, and Bach Normalization.
Colors indicate true classes.
}
\label{fig_scatter_pca}
\end{figure*}
%    \hspace{-2.0mm}\includegraphics[height=30.0mm]{fig/ScatFeatTrLDA01_ours_cl01_crop.pdf} &
%    \hspace{-5.0mm}\includegraphics[height=30.0mm]{fig/ScatFeatTrLDA01_bachnorm_cl01_crop.pdf} &
%    \hspace{-5.0mm}\includegraphics[height=30.0mm]{fig/ScatFeatTrLDA01_dropout_cl01_crop.pdf} &
%    \hspace{-5.0mm}\includegraphics[height=30.0mm]{fig/ScatFeatTrLDA01_vanilla_cl01_crop.pdf} &
%    \hspace{-5.0mm}\includegraphics[height=30.0mm]{fig/ScatFeatTrLDA01_noisy_cl01_crop.pdf} \\
%    \multicolumn{5}{c}{(b) Training-data features projected by LDA bases,
%    constructed on the class-1 features vs the rest of the features.} \\
%    \multicolumn{5}{c}{Orange: the class-1 features, blue: the rest of the features.} \\

\textbf{Motivation.}
We now use classical component analyses to 
clarify the low-dimensional structure of the FOCA features.

In this subsection we only show the CIFAR-10 results, 
because the CIFAR-100 results are qualitatively similar.

\textbf{Principal Component Analysis (PCA).}
Figure~\ref{fig_scatter_pca}~(a) shows
scatter plots of training-data features projected
by 2D bases of the PCA that is applied to
all features.
For a given class, 
the projected FOCA features
look nearly one-dimensional, 
not a point-like, per class.
This one-dimensional characteristics is probably due to
the use of softmax normalization at the last layer, 
though we have no proof so far.
In contrast, the projected features of the other methods 
clearly span two dimensions,
roughly confined in an ellipse-like region,
for a given class.
%This observation is somehow in agreement 
%with a toy-data experiment that uses cross-entropy loss 
%with softmax activation.

\begin{comment}
\textbf{Linear Discriminant Analyses (LDA).}
Figure~\ref{fig_scatter_pca}~(b) shows 
scatter plots of training-data features projected by 2D bases produced by 
one-class-vs-rest LDA.
Only class-1-vs-rest results are shown because no significant differences are observed
when replacing class-1 by another class.
It is clear from the figure that none of these 2D features are linearly separable, and 
that no significant differences are found among the methods.
The largest generalized eigenvalues of the between-class scatter with respect to
the sum of within-class scatter, shown in the middle column of Table~\ref{table_LDA_eigenvalues},
express a kind of linear separability of given features.
Among these values, the largest value is given by Batch Normalization,
but is in fact not significantly larger than the other.
\end{comment}

\begin{table}[h]
\caption{
The results of the class-1-vs-rest LDA with normalization.
}
\label{table_LDA_eigenvalues}
    \centering
    {\small
    \begin{tabular}{|l|c|c|}
    \hline Method & Eigenvalue & Test error rate \\ \hline
    FOCA (ours) & \textbf{247.28} & \textbf{2.01\%} \\
    Plain & 5.74 & 2.71\% \\
    Noisy & 7.49 & 2.86\% \\
    Dropout & 5.81 & 2.78\% \\
    Bach Norm & 7.28 & 2.43\% \\ \hline
    \end{tabular}
    %\begin{tabular}{|l|c|c|}
    %\hline Method & Unnormalized LDA & Normalized LDA \\ \hline
    %FOCA (ours) & 4.79 & \textbf{247.28} \\
    %Bach Norm & \textbf{5.82} & 7.28 \\
    %Dropout & 4.10 & 5.81 \\
    %Plain & 3.79 & 5.74 \\
    %Noisy & 4.55 & 7.49 \\ \hline
    %\end{tabular}
    }
\end{table}

\textbf{Linear Discriminant Analyses (LDA) with normalization.}
Next, we examine the LDA on features normalized to unit lengths.
Normalization is taken based on the observation that
the 2D features in Fig.~\ref{fig_scatter_pca}~(a)
are distributed mostly along 
radial direction about a point close to the origin.
Figure~\ref{fig_scatter_pca}~(b) shows 
the 2D features 
that are normalized and projected by the LDA bases described above.
Only class-1-vs-rest results are shown because no significant differences are observed
when replacing class-1 by another class.
Here, we can observe a remarkable differences;
the projected FOCA features are linearly separable with a fairly large margin,
compared to the characteristic scales of the class-1 distribution 
or of the rest-of-the-class distribution.
The form of the feature distribution is close to point-like per class,
somewhat similar to the observation in the toy experiment 
shown in Fig.~\ref{fig_toy}~(b).
%except that the projected features 
%look a little more spread in the real-dataset experiment.
In contrast, the other methods exhibit linearly non-separable feature distributions.

\textbf{Linear separability by the LDA with normalization.}
The generalized eigenvalue computed in the LDA discussed above are 
given in Table~\ref{table_LDA_eigenvalues}.
The values are the largest ratios of the between-class scatters to the within-class scatters after linear projections.
FOCA exhibits orders of magnitude larger generalized eigenvalue
than other methods.
This 
supports the high level of linear separability of the FOCA features.
Figure~\ref{fig_scatter_pca}~(c) shows the normalized features of test data, projected 
by the same LDA bases,
to see the generalizability.
Table~\ref{table_LDA_eigenvalues} also shows the smallest possible
test error rates (class-1 vs. rest)
by setting a threshold 
along the principal axis.
FOCA yields the lowest test error rate.

%%%%%%%%%%%%%%%%%%%%%%%%%%%%%%%%%%%%%%%%%%%%%%%%%%%%%%%%%%%%%%%%%%%%%%%%%%%%%%%%
\section{Conclusion}
\label{five}

A na\"ive joint optimization of 
a feature extractor and a classifier in a neural network
often brings cases where
both sets of parameters are tied in a so complex way that
the classifier is irreplaceable 
without degrading the test performance.
We introduced a method called
Feature-extractor Optimization through Classifier Anonymization (FOCA),
that is designed to break unwanted inter-layer co-adaptation.
FOCA produces a feature extractor that 
does not explicitly adapt to a particular classifier.
We gave a mathematical proposition that guarantees a simple form of 
feature distribution under special conditions;
indeed, features form a point-like distribution in a class-separable way.
Different kinds of real-dataset experiments under more general conditions 
provide supportive evidences.
%A simple projection of the high-dimensional features exhibits
%a well-separated, nearly point-like distribution property.

\bibliography{ICML2019_Sato_etal}
\bibliographystyle{icml2019}

\end{document}